\pdfoutput=1
\documentclass[11pt]{article}
\usepackage{fullpage}
\IfFileExists{arxiv/arxiv.tex}
  {\newcommand{\arxivroot}{arxiv/}}
  {\newcommand{\arxivroot}{}}
\usepackage{times}
\usepackage[T1]{fontenc}
\usepackage{amsmath}
\usepackage{amssymb}
\usepackage{amsthm}
\usepackage{natbib}
\usepackage{xcolor}
\usepackage[algo2e,ruled]{algorithm2e}
\usepackage{caption}
\usepackage{dsfont}
\usepackage{microtype}
\usepackage{enumitem}
\usepackage{placeins}
\usepackage{pgfplots}
\usepgfplotslibrary{groupplots}
\usetikzlibrary{patterns}
\pgfplotsset{compat=1.18}
\usepackage{aliascnt}
\usepackage{hyperref}
\usepackage[capitalise,noabbrev]{cleveref}
\hypersetup{colorlinks,
            linkcolor=blue,
            citecolor=blue,
            urlcolor=magenta,
            linktocpage,
            plainpages=false}
\bibpunct{(}{)}{;}{a}{,}{,}
\definecolor{experimentBlue}{HTML}{236CB1}
\pgfdeclarepatternformonly{sparse hatch}{\pgfpointorigin}{\pgfqpoint{5pt}{5pt}}{\pgfqpoint{5pt}{5pt}}{\pgfsetlinewidth{0.45pt}\pgfpathmoveto{\pgfqpoint{-0.5pt}{-0.5pt}}\pgfpathlineto{\pgfqpoint{5.5pt}{5.5pt}}\pgfusepath{stroke}}
\pgfplotsset{
  experiment axis/.style={
    scale only axis,
    width=0.249\textwidth,
    height=3.05cm,
    font=\footnotesize,
    tick label style={font=\footnotesize},
    label style={font=\footnotesize},
    title style={font=\footnotesize},
    axis lines=left,
    axis line style={black!55},
    tick style={black!55},
    grid=major,
    grid style={black!12},
    major tick length=2pt,
    scaled ticks=false,
    clip=true,
    legend style={font=\footnotesize,draw=none,fill=none,
      inner sep=1pt,cells={anchor=west}},
    legend image post style={mark size=1.5pt},
  },
  experiment series/.style={line width=0.85pt,mark size=1.6pt},
  experiment guide/.style={black!55,dashed,line width=0.7pt,no marks},
}
\pgfplotsset{
  discard if not/.style 2 args={
    x filter/.append code={
      \edef\tempa{\thisrow{#1}}
      \edef\tempb{#2}
      \ifx\tempa\tempb
      \else
        
      \fi
    },
  },
}
\definecolor{blueOne}{HTML}{08306B}
\definecolor{blueTwo}{HTML}{10448A}
\definecolor{blueThree}{HTML}{1C5CA3}
\definecolor{blueFour}{HTML}{2E74BC}
\definecolor{blueFive}{HTML}{4A8ECF}
\definecolor{blueSix}{HTML}{6AA8DD}
\definecolor{blueSeven}{HTML}{86BAE6}
\definecolor{blueEight}{HTML}{9CC9EC}
\pgfplotsset{
  swatch legend/.style={
    legend image code/.code={\draw[##1] (0cm,-0.06cm) rectangle (0.3cm,0.06cm);},
  },
}

\newcommand{\X}{\mathcal{X}}
\newcommand{\F}{\mathcal{F}}
\newcommand{\Pcal}{\mathcal{P}}
\newcommand{\Zcal}{\mathcal{Z}}
\newcommand{\Lcal}{\mathcal{L}}
\newcommand{\E}{\mathbb{E}}
\newcommand{\R}{\mathbb{R}}
\newcommand{\1}{\mathds{1}}
\newcommand{\eps}{\varepsilon}
\newcommand{\Enc}{\mathsf{Enc}}
\newcommand{\Dec}{\mathsf{Dec}}
\newcommand{\dec}{\mathsf{dec}}
\newcommand{\err}{\mathrm{err}}
\newcommand{\minerr}{\mathrm{err}}
\newcommand{\CC}{\mathrm{CC}}
\newcommand{\pub}{\mathrm{pub}}
\newcommand{\ev}{\mathrm{ev}}
\newcommand{\Gdual}{G^{*}_{\F}}
\newcommand{\Gprim}{G_{\F}}
\newcommand{\MC}{\mathrm{MC}}
\newcommand{\Var}{\operatorname{Var}}
\DeclareMathOperator{\VCdim}{VCdim}
\DeclareMathOperator{\Pdim}{Pdim}
\DeclareMathOperator{\range}{range}

\newcommand{\StudyOutcomesNoCandidate}{91}
\newcommand{\StudySelectedOutcomes}{479}
\newcommand{\StudyCertifiedOutcomes}{474}
\newcommand{\StudyEndpointInstances}{18}
\newcommand{\StudySaturationInstances}{15}
\newcommand{\StudyPilotModels}{15}
\newcommand{\StudyPilotRepresentations}{75}

\newtheoremstyle{manuscript}{\topsep}{\topsep}{\itshape}{}
  {\bfseries}{}{0.5em}
  {\thmname{#1}\thmnumber{ #2}\thmnote{ (#3)}}
\theoremstyle{manuscript}
\newtheorem{theorem}{Theorem}[section]
\newtheorem{example}{Example}[section]
\newtheorem{proposition}{Proposition}[section]
\newaliascnt{lemma}{theorem}
\newtheorem{lemma}[lemma]{Lemma}
\aliascntresetthe{lemma}
\newaliascnt{remark}{theorem}
\newtheorem{remark}[remark]{Remark}
\aliascntresetthe{remark}
\newaliascnt{corollary}{theorem}
\newtheorem{corollary}[corollary]{Corollary}
\aliascntresetthe{corollary}
\newaliascnt{definition}{theorem}
\newtheorem{definition}[definition]{Definition}
\aliascntresetthe{definition}
\newaliascnt{fact}{theorem}
\newtheorem{fact}[fact]{Fact}
\aliascntresetthe{fact}
\newaliascnt{question}{theorem}
\newtheorem{question}[question]{Question}
\aliascntresetthe{question}
\crefname{fact}{Fact}{Facts}
\crefname{question}{Question}{Questions}
\crefname{appendix}{Appendix}{Appendices}
\crefname{algocf}{Algorithm}{Algorithms}
\renewcommand{\qed}{\hfill\rule{1.5ex}{1.5ex}\par\bigskip}
\renewenvironment{proof}[1][Proof]
  {\par\noindent\textbf{#1.}\ \ignorespaces}
  {\qed}
\newenvironment{keywords}
  {\par\begingroup\small\noindent\textbf{Keywords:} \ignorespaces}
  {\par\endgroup}
\makeatletter
\newcommand{\titlefootnote}[1]{\begingroup
    \renewcommand{\@makefntext}[1]{\noindent##1}\@ifundefined{H@@footnotetext}{}{\let\@footnotetext\H@@footnotetext}\hypersetup{urlcolor=black}\footnotetext{#1}\endgroup
}
\makeatother

\author{Lunjia Hu \\ \small Northeastern University \and  Anish Jayant \\ \small Georgia Institute of Technology \and Peter Li \\ \small Northeastern University \and Princewill Okoroafor \\ \small Harvard University \and Abhishek Shetty \\ \small Georgia Institute of Technology}
\date{}
\title{Communication Complexity of Universal Features}

\begin{document}
\maketitle
\titlefootnote{Code: \url{https://github.com/tcs-ml/Feature-Representations}}
\begin{abstract}
Motivated by the notion of features from machine learning, we study the complexity of representing data that allows for evaluation with a given family of functions.
Concretely, we consider a deterministic encoder that observes $X\sim D$ and produces a finite encoding before a task $f\in\mathcal{F}\subseteq\{0,1\}^{\mathcal{X}}$ is specified.
For each task, a decoder must predict $f(X)$ with error at most $\varepsilon$ under $D$.
We construct an encoder that outputs a representation with $O(d^{*}\log(1/\varepsilon))$ bits, where $d^{*}$ is the dual VC dimension of $\mathcal{F}$. Our encoder greedily selects $O((d^{*}/\varepsilon)\log(1/\varepsilon))$ functions from the class and uses their joint label pattern to produce the encoding. The key idea in our upper bound is a novel potential function argument, where we compare the entropy gained from each greedy selection with the number of patterns allowed by the dual VC dimension. We present lower bounds showing that the representation length achieved by our encoder is worst-case optimal for every $d^* \ge 1$ and $\varepsilon < 1/4$ up to a universal constant factor, and instance-wise optimal for every family $\mathcal{F}$ up to an $O(\log(1/\varepsilon))$ factor. We give two ways to learn the representation from unlabeled samples, with different guarantees on the representation size and sample complexity. Experimentally, we verify the greedy construction's accuracy-size tradeoff and explore its connection to neural sparse autoencoders.
\end{abstract}

\begin{keywords}
  feature learning, dual VC dimension, one-way communication complexity
\end{keywords}

\section{Introduction}
\label{sec:intro}
Many questions in learning, computation, and statistics share one form: compress a data point while preserving its relationship to a fixed family of functions, before anyone knows which function will be asked.
Here are three examples:

\begin{itemize}[leftmargin=1.5em]
  \item \emph{Features for a hypothesis class.}
        A data provider releases compressed features of its inputs, and users later train classifiers from a known class $\mathcal H$, such as decision stumps or linear classifiers, for labels the provider never sees.
        The provider wants every user to do almost as well as if the raw inputs had been released.
        Compressed or hashed features shared across many learning problems are instances of this application \citep{weinberger2009feature,zhou2009compressed}.
  \item \emph{Representations for downstream tasks.}
        A pretrained model maps each input to a representation that is later probed for many tasks \citep{alain2017understanding,radford2021learning}.
        In a sparse autoencoder, for instance, an activation of a language model is represented by which of $W$ dictionary features are active, with at most $k$ active at once, and a researcher later asks whether a particular feature is active \citep{bricken2023monosemanticity,cunningham2023sparse,gao2024scaling}.
        The number $W$ of features may be in the millions \citep{templeton2024scaling}, yet intuition suggests that the sparsity $k$ should govern how much must be stored.
  \item \emph{Sketches for queries.}
        A phone knows its position $x$ in the plane and may send only a few bits to a server.
        Later, the server must decide whether the phone was inside a region, where the regions are half-planes from a large, possibly infinite, list of boundaries, as in geofencing \citep{garzon2014geofencing}.
        Location data are often released with reduced precision \citep{andres2013geo}, and answering such membership queries is the classical problem of range searching \citep{agarwal1999geometric}.
        Sending the exact position is wasteful, and sending one bit per region is impossible when the list is long.
\end{itemize}

In each case the data compression is produced before the downstream function is chosen, so it must keep enough information to support every function of the family.
Recording the value of every function is sufficient, but it is infeasible when there are infinitely many functions in the family. Recording each data point itself defeats the purpose of compression.
What is the minimum amount of information needed in the compressed representation so that we can still evaluate every downstream function in a given family with high accuracy?

We study this question in its simplest form.
An encoder observes an input $X$ drawn from a known distribution $D$ on a domain $\X$, and outputs a representation of $b$ bits.
A decoder later receives the representation together with a function $f$ from a known class $\F\subseteq\{0,1\}^{\X}$, and must predict $f(X)$.
We call the encoder-decoder pair a \emph{universal representation} of $\F$ with error $\eps$ if
\begin{equation}
  \label{eq:intro-guarantee}
  \sup_{f\in\F}\ \Pr_{X\sim D}\bigl[\Dec(\Enc(X),f)\ne f(X)\bigr] \le \eps .
\end{equation}
Here $\Enc(X)\in\{0,1\}^{b}$ is the representation, $\Dec(\Enc(X),f)\in\{0,1\}$ is the prediction for $f$, and the supremum says that a single representation must serve every member of $\F$.
We write $\CC^{\to}_{\det}(\F,D,\eps)$ for the least $b$ for which a universal representation with error $\eps$ exists.
In \Cref{sec:views} we show that this quantity is a one-way communication complexity for a communication problem between the encoder and the decoder. It is also the logarithm of the minimum size of a partition of $\X$ on which every member of $\F$ is nearly constant. Furthermore, it is the minimum length of a feature map that loses at most $\eps$ for agnostic learning with $\F$.

Prior work has focused on a randomized variant of $\CC^{\to}_{\det}(\F,D,\eps)$, where we allow the encoder and decoder to share a random seed $\omega$, drawn in advance of the data point $X\sim D$ and the downstream task $f\in \F$.
This corresponds to the \emph{public-coin} setting in communication complexity. In this setting, our error goal takes an expectation over the random seed:
\begin{equation}
  \label{eq:intro-public-guarantee}
  \sup_{f\in\F}\ \Pr_{\omega,\,X\sim D}\bigl[\Dec_\omega(\Enc_\omega(X),f)\ne f(X)\bigr]\le\eps .
\end{equation}
Let $\CC^{\to}_{\pub}(\F,D,\eps)$ be the least representation length that achieves this public-coin error goal using an arbitrary pair of encoder and decoder with an arbitrary random seed.
The following result gives a characterization of this public-coin communication complexity using the \emph{dual VC dimension} of $\F$ (\Cref{def:dual-vc-dimension}).
\begin{proposition}[\citealp{kremer1999randomized,molinaro2015amplification,bottesch2015correlation}]
\label{prop:intro-1}
For every function family $\F$ with finite \emph{dual VC dimension} $d^*\ge 1$ and every $\varepsilon\in (0,1/4)$,
    \begin{equation}
  \label{eq:intro-public-upper}
  (1-h_2(\eps))d^{*} \le  \sup_D \CC^{\to}_{\pub}(\F,D,\eps)
  \le O\!\left(d^{*}\log\frac1\eps\right),
\end{equation}
where the supremum is over all distributions $D$ over the domain $\X$ and $h_2(u)=-u\log_2 u-(1-u)\log_2(1-u)$ is the binary entropy.
\end{proposition}

This result gives a characterization of the public-coin communication complexity $\CC^{\to}_\pub$ in terms of the dual VC dimension $d^*$.
Moreover, the packing characterization of \citet{molinaro2015amplification} implies that for every $d^{*}\ge1$ and every sufficiently small $\eps$, \emph{some} class $\F$ of dual VC dimension $d^{*}$ and some distribution $D$ require $\Omega(d^{*}\log(1/\eps))$ bits even with shared randomness.
Thus the $O(d^*\log(1/\varepsilon))$ upper bound in \eqref{eq:intro-public-upper} is worst-case optimal up to a universal constant for every $d^* \ge 1$ and sufficiently small $\varepsilon$.

Our focus is on the more challenging deterministic setting \eqref{eq:intro-guarantee}, where prior work provides much weaker upper bounds.
Results from \citet{adams2012uniform} imply that every family $\F$ of finite VC dimension $d$ satisfies $\CC_{\det}^{\to} (\F,D,\varepsilon) < +\infty$ for every distribution $D$ and every $\varepsilon > 0$, but they do not provide a uniform upper bound that holds simultaneously for all distributions $D$. Such a uniform  upper bound can be inferred from the following stronger result:
\begin{proposition}
\label{prop:intro-2}
For every family $\F$ with VC dimension $d \ge 1$ and dual VC dimension $d^* \ge 1$, for every distribution $D$ on the domain $\X$ and every $\varepsilon \in (0,1/4)$, we have
\begin{equation}
\label{eq:intro-suboptimal}
\CC_{\det}^{\to} (\F,D,\varepsilon) \le O(dd^*\log(1/\varepsilon)).
\end{equation}
\end{proposition}
This result follows from taking an $\varepsilon$-covering of $\F$ and then using the values of these functions to form the encoding. It is known that $d = O(2^{d^*})$ holds for all families $\F$, so \Cref{prop:intro-2} implies a uniform upper bound $O(2^{d^*}d^*\log(1/\varepsilon))$ on $\CC_{\det}^{\to}$ that holds for all distributions $D$. However, this bound is much weaker than its randomized counterpart in \eqref{eq:intro-public-upper}.

\subsection{Our Contributions}
\label{sec:main-result}
We show that the same $O(d^*\log(1/\varepsilon))$ upper bound for public-coin randomized communication complexity also holds in the deterministic setting.
This removes the additional factor of $d$ from  \eqref{eq:intro-suboptimal}.
Thus a single deterministic encoder-decoder pair attains the same bound while satisfying the error guarantee \eqref{eq:intro-guarantee} for every downstream function.

\newcommand{\mainthmbody}{Let $\F\subseteq\{0,1\}^{\X}$ be a class of measurable functions with dual VC dimension $d^{*}$, where $1\le d^{*}<\infty$, and let $0<\eps\le1/4$.
  For every distribution $D$ on $\X$, there is a deterministic universal representation of $\F$ with error at most $\eps$, in the sense of \eqref{eq:intro-guarantee}, of length
  \[
    O\!\left(d^{*}\log\frac1\eps\right).
  \]
  It is determined by the values of $O((d^{*}/\eps)\log(1/\eps))$ members of $\F$, and \Cref{alg:greedy} constructs it.
}
\begin{theorem}[Short deterministic universal representations]
  \label{thm:main}
\mainthmbody
\end{theorem}

Our proof uses a novel, very simple greedy construction: add a function $g\in \F$ whenever its value cannot yet be predicted with error at most $\eps$ from the functions already selected. The encoder stores an index of the joint label pattern $(g_1(X),\ldots,g_t(X))$ of the selected functions, focusing only on the patterns that actually occur.
The dual VC dimension bounds the number of bits needed to encode these patterns via Sauer's lemma.
A key idea in our analysis is to use the entropy of the pattern distribution as a potential function. On one hand, each selected function $g_i$ increases this potential by at least a fixed amount, so the potential grows at least \emph{linearly} in the number $t$ of selected functions. On the other hand, the entropy potential is trivially upper bounded by the logarithm of the total number of occurring label patterns. This upper bound grows \emph{logarithmically} in $t$ by Sauer's lemma. Concretely, when $t \ge d^*$,
\[
2\varepsilon t < H(g_1(X),\ldots,g_t(X))
\le d^*\log(et/d^*).
\]
Therefore, the entire process must end before $t$ gets too large to allow the above inequality. This gives an upper bound on the total number of selected functions as well as the communication complexity. See \Cref{sec:tight} for our full proof of \Cref{thm:main}.

The simplicity of this construction is a major advantage of our result over \Cref{prop:intro-1}, where the encoder-decoder pair comes from a non-constructive minimax argument (\Cref{sec:public}).
Our proof explicitly constructs a deterministic encoder-decoder pair using a simple greedy algorithm.

The previous public-coin lower bounds discussed earlier naturally apply to our deterministic setting. They show that our upper bound is optimal in the worst case over classes and distributions.
In particular, the packing characterization of \citet{molinaro2015amplification} gives classes requiring $\Omega(d^{*}\log(1/\eps))$ bits even in the easier public-coin setting.
For every fixed family $\F$, the lower bound in \Cref{prop:intro-1} gives a distribution $D$ under which at least $(1-h_2(\eps))d^{*}$ bits are necessary.
We include self-contained proofs of these lower bounds in \Cref{sec:instance-lower,sec:threshold}.

We also study how to construct universal representations when $D$ is unknown.
In the \emph{shared-sample} model of \Cref{sec:samples}, encoder and decoder are built from $n$ unlabeled samples of $D$ and must then work on a fresh input.
The next theorem gives two guarantees: the first retains our $O(d^{*}\log(1/\eps))$ length bound, while the second learns a representation with the length bound of \Cref{prop:intro-2} using fewer samples.
Here $d$ and $d^{*}$ denote the primal and dual VC dimensions, as above.

\newcommand{\samplethmbody}{Let $\F\subseteq\{0,1\}^{\X}$ have primal VC dimension $d$ and dual VC dimension $d^{*}$, with $1\le d,d^{*}<\infty$, and let $0<\eps\le1/4$ and $0<\delta<1$.
  Given independent samples from an unknown distribution $D$ on $\X$, one can learn, with probability at least $1-\delta$, a universal representation of $\F$ with error at most $\eps$ under $D$ in either of the following ways.
  \begin{enumerate}[label=(\alph*),leftmargin=2em]
    \item \emph{Optimal length.} \Cref{alg:sample-greedy} outputs a representation of length $O(d^{*}\log(1/\eps))$ from
          \[
            \left(\frac{C}{\eps}\right)^{d^{*}+1}\left(\log\frac1\eps\right)^{d^{*}+2}+\frac{C}{\eps}\log\frac1\delta
          \]
          samples, where $C$ is an absolute constant.
    \item \emph{Few samples.} \Cref{alg:cover} outputs a representation of length $O(d\,d^{*}\log(1/\eps))$ from
          \[
            O\!\left(\frac{d\log(1/\eps)+\log(1/\delta)}{\eps}\right)
          \]
          samples.
  \end{enumerate}
}
\begin{theorem}[Learning universal representations from samples]
  \label{thm:samples-intro}
\samplethmbody
\end{theorem}

The two parts trade length against samples.
Part (a) keeps the optimal length $O(d^{*}\log(1/\eps))$ of \Cref{thm:main}, at the price of a sample size exponential in $d^{*}$ (\Cref{thm:sample-optimal}).
Part (b) uses only as many samples as are known to suffice for PAC learning a single function of $\F$ \citep{vapnik1971uniform,anthony1993result}, at the price of a factor $d$ in the length.
Whether one can have both is open; even for thresholds on a line, learning a representation of near-optimal length requires $\Omega((1/\eps)\log(1/\eps))$ samples (\Cref{thm:sample-lower}).
For thresholds, this sample lower bound is matched at constant confidence by a learner with length $\log(1/\eps)+O(1)$ (\Cref{cor:threshold-samples}).
\subsubsection{Additional Results}
\paragraph{Refined Complexity Bounds.}
We sharpen our upper bound using the dual growth function, which counts the label patterns realized by selected functions, and extend the classical entropy lower bound to obtain class-specific bounds under suitable input distributions (\Cref{thm:growth-upper,prop:growth-lower}).
We also characterize zero-error complexity for finitely supported distributions and give matching worst-case bounds for finite function classes (\Cref{sec:interpolation}).
For representations generated by members of $\F$, we prove a worst-case $\Omega(d^{*}/\eps)$ lower bound on the number of selected functions, within a logarithmic factor of our construction (\Cref{sec:feature-count}).

\paragraph{Reconstruction objectives.}
On finite input domains, we show that every maximizer of joint entropy over sets of at most $t$ functions from $\F$ gives a universal representation when $t$ is sufficiently large (\Cref{sec:one-shot}).
We connect this objective to optimal reconstruction log loss (\Cref{sec:reconstruction}).
We also show that minimizing reconstruction log loss separately for each input coordinate can fail to preserve all tasks under a selection budget, even when the dual VC dimension is two (\Cref{sec:mi}).

\paragraph{Learning interpretations.}
We characterize universality by the ability to achieve prediction error at most $\eps$ above the best error in $\F$ for every downstream labeling (\Cref{prop:agnostic}).
We also prove an equivalence, up to a factor of two in error, with multicalibration of the conditional-mean predictor $\E[Y\mid\Enc(X)]$ under every labeling.
This equivalence uses error averaged over prediction values and weighted by group probability (\Cref{sec:multicalibration}).

\paragraph{Extensions.}
We extend our deterministic bound to tasks and predictions in $[0,1]$ under expected absolute loss, obtaining $O(d^{*}_{\mathrm P}\log(1/\eps))$ bits, where $d^{*}_{\mathrm P}$ is the pseudo-dimension of the dual class (\Cref{sec:real}).
We also recall the public-coin characterization of \citet{molinaro2015amplification} and bound the number of protocols needed in a randomized mixture (\Cref{sec:public}), give a two-message protocol (\Cref{sec:interaction}), and describe a halfspace representation based on the partition theorem of \citet{matousek1992efficient} (\Cref{sec:halfspaces}).

\paragraph{Experiments and open questions.}
We test our greedy construction on threshold products and explore finite-family saturation and accuracy--length tradeoffs in reconstruction-trained sparse autoencoders (\Cref{sec:experiments}).
We give experimental configurations in \Cref{app:experiments} and discuss open questions in \Cref{sec:discussion}.

\subsection{Related work}
\label{sec:related}

\paragraph{One-way communication under product distributions.}
The relevant communication matrix is $M(x,f)=f(x)$, with the input holder sending a message to the function holder.
Its rows are the evaluation functions $\ev_x:f\mapsto f(x)$, so its row VC dimension is exactly $d^{*}$.
\citet{kremer1999randomized} characterize constant-error one-way distributional complexity, maximized over product distributions, by this dimension.
\citet{molinaro2015amplification} give the more precise characterization at varying error: complexity is controlled, up to constant factors and a constant rescaling of the error, by the logarithm of the largest packing number of the rows, maximized over distributions on columns.
For our matrix, the underlying pseudometric is
\[
  \rho_Q(x,x')=\Pr_{f\sim Q}[f(x)\ne f(x')],
\]
where $Q$ is a distribution on tasks.
Combining their characterization with VC packing bounds \citep{haussler1995sphere} gives $O(d^{*}\log(1/\eps))$ bits.
\citet{bottesch2015correlation} independently obtain logarithmic error amplification for one-way communication under product distributions.
Thus both the constant-error role of $d^{*}$ and the logarithmic accuracy dependence of the public-coin upper bound are prior results.

The quantifiers explain the connection to universal representations.
For a fixed $D$, the product-distribution statement gives, for each $Q$, a deterministic protocol with small error averaged over $D\times Q$; the protocol may depend on $Q$.
For finite classes and domains, minimax \citep{yao1977probabilistic} converts these protocols into one public-coin protocol whose error, averaged over its seed and $X\sim D$, is at most $\eps$ for every $f$.
This does not select a single seed for which the guarantee holds for all tasks.
The public-coin statements in \Cref{sec:public} are included as consequences of this literature, especially the packing characterization of \citet{molinaro2015amplification}.
Their greater-than example also supplies the one-dimensional logarithmic lower bound.
Our contribution is a single deterministic representation attaining this length bound and generated by functions from the class; to the best of our knowledge, this guarantee has not appeared previously.
The matching lower bound is a worst-case statement over classes and distributions, not a claim that every fixed class requires $\Omega(d^{*}\log(1/\eps))$ bits.

\paragraph{Correlated inputs and general joint distributions.}
Allowing an arbitrary joint distribution $\mu$ on $(X,f)$ changes the comparison: a distributional protocol minimizes error under $\mu$, whereas our guarantee fixes one input distribution and controls each task separately.
\citet{jain2009new} extend the VC upper bound beyond product distributions, with a dependence on the mutual information between the two inputs; at constant error their bound specializes to $O((I_\mu(X;f)+1)d^{*})$.
\citet{bottesch2015correlation} study this interpolation systematically and improve the accuracy dependence of the correlated one-way bound.
\citet{sherstov2010communication} shows that restricting hard distributions to be product can dramatically reduce communication complexity, including in the one-way model.
Consequently, the product-distribution characterization should not be read as a characterization of worst-case randomized communication over arbitrary input pairs.
Nor does small error under one correlated distribution imply small error for every task under the common marginal $D$.

\paragraph{Fixed-distribution covers and representative domains.}
Approximation by finite covers in the disagreement metric $\Pr_D[f(X)\ne g(X)]$ goes back to \citet{benedek1991learnability}.
For our purpose, a proper $\eps$-cover $G\subseteq\F$ gives a direct construction: partition $\X$ by the vector $(g(x))_{g\in G}$ and predict each $f$ using a nearby cover member.
Every task then has error at most $\eps$ under $D$.
Haussler's packing bound gives $|G|\le(C/\eps)^d$, and dual Sauer--Shelah bounds the number of realized patterns by $\sum_{i=0}^{\min\{|G|,d^{*}\}}\binom{|G|}{i}$.
The resulting code has length $O(d\,d^{*}\log(1/\eps))$.

This partition is also the \emph{representative-domain} construction used by \citet{bassily2020private} for private query release assisted by public data.
They retain one input per label pattern of a cover, bound the number of representatives using dual VC dimension, and map data points to their representatives.
\citet{bendavid2023private} use the same cover and representative-domain construction in private distribution learning with public data.
Our cover-based sample bound should therefore be viewed as an application of established tools.
The entropy-greedy construction instead selects functions according to their remaining error on the current partition, removing the extra factor $d$ from the representation length.

\paragraph{Finite approximation by partitions.}
\citet{adams2012uniform} prove finite approximability of VC classes by partitions under every probability measure.
They control the total mass of cells on which a set and its complement both have positive measure, a stronger condition than small minority mass.
Their theorem establishes finiteness without an explicit bound on the number of cells; our upper bound controls this number for the majority-error criterion.

We discuss broader connections to learning and representation models in \Cref{app:additional-related}.

\section{Preliminaries}
\label{sec:prelims}

\paragraph{Notation and measurability.}
For an integer $n\ge1$, write $[n]=\{1,\ldots,n\}$.
All logarithms are base two unless written as $\ln$.
Let $(\X,\Sigma,D)$ be a probability space, and let $\F\subseteq\{0,1\}^{\X}$ be a nonempty class of measurable functions.
For results involving samples, we assume the event classes and learning rules satisfy the usual measurability conditions for the stated uniform-convergence bounds.
Finite classes on finite domains satisfy these conditions automatically.

\subsection{VC dimension and the dual class}
\label{sec:dual}

VC dimension measures how many inputs can be labeled independently by functions in the class.
Dual VC dimension reverses the roles of inputs and functions.

\begin{definition}[VC dimension]
  \label{def:vc-dimension}
  A finite set of inputs $\{x_1,\ldots,x_n\}\subseteq\X$ is \emph{shattered} by $\F$ if
  \[
    \bigl\{(f(x_1),\ldots,f(x_n)):f\in\F\bigr\}=\{0,1\}^n.
  \]
  The \emph{VC dimension} $d=\VCdim(\F)$ is the supremum of the sizes of shattered sets, with value $+\infty$ if these sizes are unbounded.
  We also call $d$ the \emph{primal VC dimension} to distinguish it from the dual dimension.
\end{definition}

For each input $x\in\X$, define its \emph{evaluation function}
\[
  \ev_x:\F\to\{0,1\},\qquad \ev_x(f)=f(x).
\]
An evaluation function records the answers to all tasks at a fixed input.
The collection of these functions is a class on the domain $\F$ itself.

\begin{definition}[Dual class and dual VC dimension]
  \label{def:dual-vc-dimension}
  The \emph{dual class} of $\F$ is $\F^{*}=\{\ev_x:x\in\X\}$, and its VC dimension $d^{*}=\VCdim(\F^{*})$ is the \emph{dual VC dimension} of $\F$.
  Equivalently, $d^{*}$ is the supremum of the sizes of families $\{f_1,\ldots,f_t\}\subseteq\F$ that are \emph{dual-shattered}, meaning that
  \[
    \bigl\{(f_1(x),\ldots,f_t(x)):x\in\X\bigr\}=\{0,1\}^{t}.
  \]
\end{definition}

Primal shattering fixes inputs and varies the function; dual shattering fixes functions and varies the input.
In the evaluation matrix $M(x,f)=f(x)$, the rows are the members of $\F^{*}$ and the columns are the members of $\F$.
\citet{assouad1983densite} showed that
\[
  d^{*}<2^{d+1},\qquad d<2^{d^{*}+1}.
\]
Thus one dimension is finite if and only if the other is, although their values can differ substantially.
Throughout the binary results we assume $1\le d^{*}<\infty$.
If $d^{*}=0$, every function is constant on $\X$.
Statements involving sample bounds in terms of $d$ also assume $d\ge1$.

\subsection{Dual growth}
\label{sec:dual-growth}

VC dimension records when all binary patterns can occur.
The growth function keeps track of the number of patterns at every size, including beyond the dimension.
For a finite list $\mathbf g=(g_1,\ldots,g_t)$ of members of $\F$, define the \emph{pattern map}
\[
  Z_{\mathbf g}(x)=(g_1(x),\ldots,g_t(x)).
\]

The \emph{dual growth function} is
\[
  \Gdual(t)=\sup_{g_1,\ldots,g_t\in\F}
  \bigl|\{Z_{\mathbf g}(x):x\in\X\}\bigr|,
  \qquad \Gdual(0)=1.
\]
Lists may repeat functions, so the definition also covers $t>|\F|$ when the class is finite.
The supremum is a maximum because the number of patterns is an integer between $1$ and $2^t$.
For comparison, the \emph{primal growth function} is
\[
  \Gprim(n)=\sup_{x_1,\ldots,x_n\in\X}
  \bigl|\{(f(x_1),\ldots,f(x_n)):f\in\F\}\bigr|,
  \qquad \Gprim(0)=1.
\]
For positive integers $t$, $\Gdual(t)=2^t$ exactly when $t\le d^{*}$, and the corresponding statement holds for $\Gprim$ and $d$.

The Sauer--Shelah lemma \citep{sauer1972density,shelah1972combinatorial}, applied to the dual class, gives
\begin{equation}
  \label{eq:sauer}
  \Gdual(t)\le\sum_{i=0}^{\min\{t,d^{*}\}}\binom ti,
  \qquad
  \Gdual(t)\le(et/d^{*})^{d^{*}}\quad(t\ge d^{*}).
\end{equation}
Similarly,
\[
  \Gprim(n)\le\sum_{i=0}^{\min\{n,d\}}\binom ni,
  \qquad
  \Gprim(n)\le(en/d)^d\quad(n\ge d\ge1).
\]

\subsection{Examples of dual dimension and growth}
\label{sec:dimension-examples}

The following examples from basic function classes illustrate what the dual parameters measure.

\begin{example}[Dual VC dimension of common classes]
  \label{ex:classes}
  \leavevmode
  \begin{enumerate}[label=(\roman*),leftmargin=2em]
    \item \emph{Thresholds.}
          For $\X=[N]$ with $N\ge2$ and $\F=\{\1\{x\le a\}:a\in[N]\}$, the functions are nested, so no two are dual-shattered and $d^{*}=1$; also $d=1$.
          A list of $t$ thresholds cuts the line into at most $t+1$ nonempty intervals, and $\Gdual(t)=\min\{t+1,N\}$.
    \item \emph{Coordinates and parities.}
          For $r\ge1$ and $\X=\{0,1\}^{r}$, both the coordinate class $\{x\mapsto x_i:i\in[r]\}$ and the parity class $\{x\mapsto w^{\top}x\bmod2:w\in\{0,1\}^{r}\}$ have $d^{*}=r$.
          The $r$ coordinate functions are dual-shattered, while $2^r$ inputs cannot realize all patterns on $r+1$ functions.
          Both classes have $\Gdual(t)=2^{\min\{t,r\}}$.
    \item \emph{All labelings.}
          If $\X$ has $M\ge1$ elements and $\F=\{0,1\}^{\X}$, then $\Gdual(t)=\min\{2^t,M\}$ and $d^{*}=\lfloor\log M\rfloor$, while $d=M$.
          The upper bound on growth counts inputs and binary patterns; it is attained by assigning distinct patterns to as many inputs as possible.
    \item \emph{Sparse inputs.}
          For $\X=\{x\in\{0,1\}^{W}:\|x\|_0\le k\}$ with $1\le k\le W$ and the coordinate class, $d^{*}=k$.
          Any $k$ coordinates realize all patterns, whereas $k+1$ coordinates cannot all equal one.
          Writing $s=\min\{t,W\}$, the dual growth is exactly $\Gdual(t)=\sum_{i=0}^{\min\{k,s\}}\binom si$.
  \end{enumerate}
\end{example}

The sparse-input example shows how a bound on the number of active coordinates yields a small dual VC dimension even when the number of available coordinates is large.
Rare activation alone does not give this conclusion: independent binary coordinates with activation probabilities strictly between zero and one still have full support and dual VC dimension $W$.

\subsection{Entropy facts}
\label{sec:entropy}

For $p\in[0,1]$, write $h_2(p)=-p\log p-(1-p)\log(1-p)$, with $h_2(0)=h_2(1)=0$.
For finite-valued random variables, $H(U)$ denotes Shannon entropy in bits and $H(U\mid V)$ denotes conditional entropy.
We use the following standard facts \citep{cover2006elements}.

\begin{fact}[Entropy facts]
  \label{fact:entropy}
  \leavevmode
  \begin{enumerate}[label=(\roman*),leftmargin=2em]
    \item $H(U,V)=H(U)+H(V\mid U)$ and $H(U)\le\log|\range(U)|$.
    \item Conditioning cannot increase entropy: $H(U\mid V,W)\le H(U\mid V)\le H(U)$.
    \item $h_2(p)\ge2\min\{p,1-p\}$ for every $p\in[0,1]$.
    \item For integers $n\ge1$ and $0\le\alpha\le1/2$, $\sum_{j=0}^{\lfloor\alpha n\rfloor}\binom nj\le2^{h_2(\alpha)n}$.
  \end{enumerate}
\end{fact}

Item (iii) follows from concavity of binary entropy on $[0,1/2]$ and symmetry about $1/2$.
It connects prediction error to the entropy gained by revealing one more label: a label that is hard to predict must contribute conditional entropy.

\section{Universal representations}
\label{sec:model}
\suppressfloats[t]

Let $\F$ be a family of binary functions on $\X$, and let $D$ be a distribution on $\X$.
Each function $f\in\F$ describes a prediction task: given an input $x$, predict the label $f(x)$.
We would like to replace $x$ by a short representation that can be used for any of these tasks.
An encoder produces this representation from $x$ before the task is specified.
A decoder later receives the representation and a function $f\in\F$, and uses them to predict $f(x)$.
The same encoder must support every function in the family, and we measure each function's prediction error on inputs drawn from $D$.

\begin{definition}[Universal representation]
  \label{def:protocol}
  A \emph{universal representation} of length $b$ consists of a measurable encoder $\Enc:\X\to\{0,1\}^{b}$ and a decoder $\Dec:\{0,1\}^{b}\times\F\to\{0,1\}$.
  Both maps are deterministic.
  Its \emph{error} is
  \[
    \sup_{f\in\F}\ \Pr_{X\sim D}\bigl[\Dec(\Enc(X),f)\ne f(X)\bigr].
  \]
  It is an $\eps$-universal representation if this error is at most $\eps$.
  The \emph{universal representation complexity} $\CC^{\to}_{\det}(\F,D,\eps)$ is the least length of such a representation, with value $+\infty$ if no finite length suffices.
\end{definition}

The term ``universal'' refers to reusing one encoder across all tasks in $\F$.
For each fixed task, the decoder must be correct with probability at least $1-\eps$ under $D$.
The inputs on which it makes errors may differ from one task to another.

\subsection{Majority decoding and partitions}
\label{sec:representation-partitions}

To understand what can be predicted from a representation, it is useful to group together inputs that receive the same code.
For a fixed task $f$, the decoder makes the same prediction on all inputs in such a group.
We describe these groups using a partition of the input domain.

\begin{definition}[Finite measurable partition]
  \label{def:partition}
  A \emph{finite measurable partition} of $\X$ is a collection $\Pcal=\{A_1,\ldots,A_K\}$ of nonempty measurable subsets of $\X$ such that the sets are pairwise disjoint and their union is $\X$.
  The sets $A_i$ are called the \emph{cells} of the partition.
\end{definition}

Every input therefore belongs to exactly one cell.
An encoder defines a partition by assigning to each code $z$ the set
\[
  A_z=\{x\in\X:\Enc(x)=z\},
\]
and discarding the empty sets.
A $b$-bit encoder has at most $2^b$ such cells.
Knowing the code is equivalent to knowing which cell contains the input.

For a task $f$ and a cell $A$ with $D(A)>0$, the decoder's goal is to predict $f(X)$ given that $X\in A$.
The best prediction is the label with larger probability under this conditional distribution.
We call this \emph{majority decoding}, where majority is measured by probability mass under $D$.
Either label is optimal in a tie, and the prediction on a cell of zero probability does not affect the error.

\begin{samepage}
\begin{lemma}[Majority decoding]
  \label{lem:majority}
  Let $Y\in\{0,1\}$ with $\Pr[Y=1]=p$.
  Then
  \[
    \min_{a\in\{0,1\}}\Pr[Y\ne a]=\min\{p,1-p\},
  \]
  attained by a majority label.
\end{lemma}
\begin{proof}
  Predicting $0$ incurs error $p$, while predicting $1$ incurs error $1-p$.
\end{proof}
\end{samepage}

On a cell $A$, majority decoding makes errors on the inputs with the less likely label.
The probability mass of those inputs is the smaller of $D(\{x\in A:f(x)=0\})$ and $D(\{x\in A:f(x)=1\})$.
Summing over cells gives the \emph{majority error} of $f$ on a finite measurable partition $\Pcal$:
\[
  \minerr_{\Pcal}(f)=\sum_{A\in\Pcal}
  \min\bigl\{D(\{x\in A:f(x)=0\}),D(\{x\in A:f(x)=1\})\bigr\}.
\]
This is the smallest prediction error achievable for $f$ when the decoder is given only the cell containing the input.
The majority label may depend on $f$, while the partition is shared by all tasks.

A short universal representation therefore corresponds to a partition with few cells on which every task has small majority error.
The following proposition makes the relation between the number of cells and the number of bits precise.

\begin{proposition}[Representations are partitions]
  \label{prop:partition}
  Let $K$ be the least number of cells in a finite measurable partition satisfying $\sup_{f\in\F}\minerr_{\Pcal}(f)\le\eps$, with $K=+\infty$ if no such partition exists.
  Then $\CC^{\to}_{\det}(\F,D,\eps)=\lceil\log K\rceil$, where the right-hand side is interpreted as $+\infty$ when $K=+\infty$.
\end{proposition}
\begin{proof}
  A $b$-bit encoder defines a partition with at most $2^b$ cells.
  By \Cref{lem:majority}, using the majority label on each cell minimizes the prediction error for each task, so an $\eps$-universal representation gives a partition with $\sup_{f\in\F}\minerr_{\Pcal}(f)\le\eps$.
  Conversely, given such a partition with $K$ cells, the encoder can send the index of the cell containing $x$ using $\lceil\log K\rceil$ bits.
  For each task $f$, the decoder outputs its majority label on that cell and achieves error $\minerr_{\Pcal}(f)\le\eps$.
\end{proof}

This equivalence lets us construct and analyze representations by working with their cells.
For any finite-valued measurable map $Z$ on $\X$, we write $\minerr_Z(f)$ for the majority error on the partition that groups inputs with the same value of $Z$.
\Cref{fig:partition} illustrates the error for one task, and \Cref{sec:views} gives further interpretations in terms of queries, communication, and agnostic learning.

\begin{figure}[t]
  \centering
  \begin{minipage}[c]{0.49\textwidth}
    \centering
    \begin{tikzpicture}[x=0.54cm,y=0.54cm,line cap=round,line join=round]
      \fill[pattern=sparse hatch,pattern color=experimentBlue!80] (3.777,3.597) .. controls (3.329,4.757) .. (3.353,6.000) .. controls (1.676,5.700) .. (0.000,6.000) .. controls (0.400,4.537) .. (0.000,3.074) .. controls (1.009,2.820) .. (1.990,2.471) .. controls (3.097,2.695) .. (3.777,3.597);
      \fill[pattern=sparse hatch,pattern color=experimentBlue!80] (12.000,3.693) .. controls (11.857,4.847) .. (12.000,6.000) .. controls (9.645,6.180) .. (7.291,6.000) .. controls (7.486,5.133) .. (7.110,4.328) .. controls (8.020,3.742) .. (8.432,2.742) .. controls (10.319,2.831) .. (12.000,3.693);
      \fill[pattern=sparse hatch,pattern color=experimentBlue!80] (0.000,0.000) .. controls (1.269,0.400) .. (2.539,0.000) .. controls (2.162,1.213) .. (1.990,2.471) .. controls (1.009,2.820) .. (0.000,3.074) .. controls (0.400,1.537) .. (0.000,0.000);
      \draw[line width=0.75pt] (3.777,3.597) .. controls (3.329,4.757) .. (3.353,6.000);
      \draw[line width=0.75pt] (3.353,6.000) .. controls (1.676,5.700) .. (0.000,6.000);
      \draw[line width=0.75pt] (0.000,6.000) .. controls (0.400,4.537) .. (0.000,3.074);
      \draw[line width=0.75pt] (0.000,3.074) .. controls (1.009,2.820) .. (1.990,2.471);
      \draw[line width=0.75pt] (1.990,2.471) .. controls (3.097,2.695) .. (3.777,3.597);
      \draw[line width=0.75pt] (7.110,4.328) .. controls (7.486,5.133) .. (7.291,6.000);
      \draw[line width=0.75pt] (7.291,6.000) .. controls (5.322,5.631) .. (3.353,6.000);
      \draw[line width=0.75pt] (3.777,3.597) .. controls (4.235,3.239) .. (4.793,3.073);
      \draw[line width=0.75pt] (4.793,3.073) .. controls (5.761,4.052) .. (7.110,4.328);
      \draw[line width=0.75pt] (12.000,3.693) .. controls (11.857,4.847) .. (12.000,6.000);
      \draw[line width=0.75pt] (12.000,6.000) .. controls (9.645,6.180) .. (7.291,6.000);
      \draw[line width=0.75pt] (7.110,4.328) .. controls (8.020,3.742) .. (8.432,2.742);
      \draw[line width=0.75pt] (8.432,2.742) .. controls (10.319,2.831) .. (12.000,3.693);
      \draw[line width=0.75pt] (2.539,0.000) .. controls (4.127,-0.400) .. (5.715,0.000);
      \draw[line width=0.75pt] (5.715,0.000) .. controls (4.871,1.422) .. (4.793,3.073);
      \draw[line width=0.75pt] (1.990,2.471) .. controls (2.162,1.213) .. (2.539,0.000);
      \draw[line width=0.75pt] (5.715,0.000) .. controls (6.645,-0.169) .. (7.575,0.000);
      \draw[line width=0.75pt] (7.575,0.000) .. controls (8.276,1.286) .. (8.432,2.742);
      \draw[line width=0.75pt] (7.575,0.000) .. controls (9.787,-0.400) .. (12.000,0.000);
      \draw[line width=0.75pt] (12.000,0.000) .. controls (11.600,1.847) .. (12.000,3.693);
      \draw[line width=0.75pt] (0.000,0.000) .. controls (1.269,0.400) .. (2.539,0.000);
      \draw[line width=0.75pt] (0.000,3.074) .. controls (0.400,1.537) .. (0.000,0.000);
      \draw[line width=0.8pt] (2.696,4.312) +(-0.14,-0.14) -- +(0.14,0.14) +(-0.14,0.14) -- +(0.14,-0.14);
      \draw[line width=0.8pt] (2.852,5.162) +(-0.14,-0.14) -- +(0.14,0.14) +(-0.14,0.14) -- +(0.14,-0.14);
      \draw[line width=0.8pt] (1.008,3.525) +(-0.14,-0.14) -- +(0.14,0.14) +(-0.14,0.14) -- +(0.14,-0.14);
      \draw[line width=0.8pt] (1.312,4.794) +(-0.14,-0.14) -- +(0.14,0.14) +(-0.14,0.14) -- +(0.14,-0.14);
      \draw[line width=0.8pt] (2.171,4.842) +(-0.14,-0.14) -- +(0.14,0.14) +(-0.14,0.14) -- +(0.14,-0.14);
      \draw[line width=0.8pt] (0.566,4.507) +(-0.14,-0.14) -- +(0.14,0.14) +(-0.14,0.14) -- +(0.14,-0.14);
      \filldraw[fill=white,line width=0.6pt] (2.164,3.336) circle (0.13);
      \filldraw[fill=white,line width=0.6pt] (1.463,4.183) circle (0.13);
      \draw[line width=0.8pt] (5.198,4.637) +(-0.14,-0.14) -- +(0.14,0.14) +(-0.14,0.14) -- +(0.14,-0.14);
      \draw[line width=0.8pt] (5.907,5.340) +(-0.14,-0.14) -- +(0.14,0.14) +(-0.14,0.14) -- +(0.14,-0.14);
      \filldraw[fill=white,line width=0.6pt] (5.091,3.974) circle (0.13);
      \filldraw[fill=white,line width=0.6pt] (6.053,4.696) circle (0.13);
      \filldraw[fill=white,line width=0.6pt] (5.202,5.443) circle (0.13);
      \filldraw[fill=white,line width=0.6pt] (4.456,5.286) circle (0.13);
      \filldraw[fill=white,line width=0.6pt] (4.313,4.619) circle (0.13);
      \draw[line width=0.8pt] (8.523,3.594) +(-0.14,-0.14) -- +(0.14,0.14) +(-0.14,0.14) -- +(0.14,-0.14);
      \draw[line width=0.8pt] (8.035,4.894) +(-0.14,-0.14) -- +(0.14,0.14) +(-0.14,0.14) -- +(0.14,-0.14);
      \draw[line width=0.8pt] (8.698,4.313) +(-0.14,-0.14) -- +(0.14,0.14) +(-0.14,0.14) -- +(0.14,-0.14);
      \draw[line width=0.8pt] (11.357,5.448) +(-0.14,-0.14) -- +(0.14,0.14) +(-0.14,0.14) -- +(0.14,-0.14);
      \draw[line width=0.8pt] (9.913,3.984) +(-0.14,-0.14) -- +(0.14,0.14) +(-0.14,0.14) -- +(0.14,-0.14);
      \draw[line width=0.8pt] (9.238,3.900) +(-0.14,-0.14) -- +(0.14,0.14) +(-0.14,0.14) -- +(0.14,-0.14);
      \filldraw[fill=white,line width=0.6pt] (9.377,4.574) circle (0.13);
      \draw[line width=0.8pt] (3.746,1.265) +(-0.14,-0.14) -- +(0.14,0.14) +(-0.14,0.14) -- +(0.14,-0.14);
      \draw[line width=0.8pt] (3.235,0.629) +(-0.14,-0.14) -- +(0.14,0.14) +(-0.14,0.14) -- +(0.14,-0.14);
      \filldraw[fill=white,line width=0.6pt] (4.374,1.678) circle (0.13);
      \filldraw[fill=white,line width=0.6pt] (3.063,1.264) circle (0.13);
      \filldraw[fill=white,line width=0.6pt] (4.060,0.705) circle (0.13);
      \filldraw[fill=white,line width=0.6pt] (3.518,2.187) circle (0.13);
      \filldraw[fill=white,line width=0.6pt] (4.086,2.543) circle (0.13);
      \draw[line width=0.8pt] (7.083,2.056) +(-0.14,-0.14) -- +(0.14,0.14) +(-0.14,0.14) -- +(0.14,-0.14);
      \draw[line width=0.8pt] (6.404,3.017) +(-0.14,-0.14) -- +(0.14,0.14) +(-0.14,0.14) -- +(0.14,-0.14);
      \filldraw[fill=white,line width=0.6pt] (6.194,1.123) circle (0.13);
      \filldraw[fill=white,line width=0.6pt] (7.136,3.180) circle (0.13);
      \filldraw[fill=white,line width=0.6pt] (6.451,2.243) circle (0.13);
      \filldraw[fill=white,line width=0.6pt] (7.131,1.242) circle (0.13);
      \filldraw[fill=white,line width=0.6pt] (5.809,2.619) circle (0.13);
      \filldraw[fill=white,line width=0.6pt] (7.689,2.467) circle (0.13);
      \draw[line width=0.8pt] (10.928,1.164) +(-0.14,-0.14) -- +(0.14,0.14) +(-0.14,0.14) -- +(0.14,-0.14);
      \draw[line width=0.8pt] (11.350,1.899) +(-0.14,-0.14) -- +(0.14,0.14) +(-0.14,0.14) -- +(0.14,-0.14);
      \filldraw[fill=white,line width=0.6pt] (10.562,0.655) circle (0.13);
      \filldraw[fill=white,line width=0.6pt] (9.638,1.268) circle (0.13);
      \filldraw[fill=white,line width=0.6pt] (11.137,2.749) circle (0.13);
      \filldraw[fill=white,line width=0.6pt] (10.200,2.133) circle (0.13);
      \filldraw[fill=white,line width=0.6pt] (8.928,1.045) circle (0.13);
      \filldraw[fill=white,line width=0.6pt] (8.845,2.210) circle (0.13);
      \draw[line width=0.8pt] (0.745,0.612) +(-0.14,-0.14) -- +(0.14,0.14) +(-0.14,0.14) -- +(0.14,-0.14);
      \draw[line width=0.8pt] (1.313,1.807) +(-0.14,-0.14) -- +(0.14,0.14) +(-0.14,0.14) -- +(0.14,-0.14);
      \draw[line width=0.8pt] (0.700,1.438) +(-0.14,-0.14) -- +(0.14,0.14) +(-0.14,0.14) -- +(0.14,-0.14);
      \draw[line width=0.8pt] (1.554,0.634) +(-0.14,-0.14) -- +(0.14,0.14) +(-0.14,0.14) -- +(0.14,-0.14);
      \filldraw[fill=white,line width=0.6pt] (0.597,2.158) circle (0.13);
\end{tikzpicture}
  \par\vspace{0.5em}

  {\scriptsize
    \mbox{\tikz[baseline=-0.5ex]{\fill[white]
          (0,-0.12) rectangle (0.45,0.12);
        \path[
          draw=black,
          line width=0.5pt,
          pattern=sparse hatch,
          pattern color=experimentBlue!80
        ] (0,-0.12) rectangle (0.45,0.12);
      }\enspace Majority $\times$\qquad
      \tikz[baseline=-0.5ex]{\filldraw[
          draw=black,
          fill=white,
          line width=0.5pt
        ] (0,-0.12) rectangle (0.45,0.12);
      }\enspace Majority $\circ$}\par
  }
  \end{minipage}\hfill
  \begin{minipage}[c]{0.47\textwidth}
    \caption{A representation $\mathcal P$ with seven cells and one task $f$. Each cell is decoded by its majority label, so the minority points are the errors and their mass is $\err_{\mathcal P}(f)$. A universal representation keeps this mass at most $\eps$ for every task in $\F$ at once.}
    \label{fig:partition}
  \end{minipage}
\end{figure}
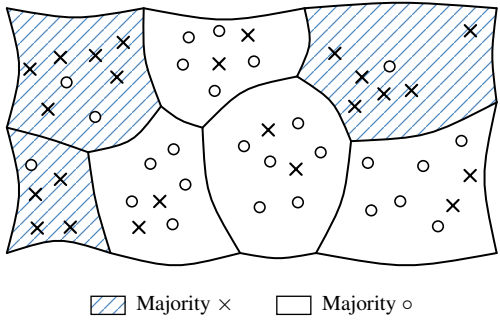

\subsection{Representations generated by functions}
\label{sec:generated-representations}

One way to build a representation is to select functions $g_1,\ldots,g_t\in\F$ and record their labels on the input.
For the list $\mathbf g=(g_1,\ldots,g_t)$, this gives the pattern map from \Cref{sec:dual-growth},
\[
  Z_{\mathbf g}(x)=(g_1(x),\ldots,g_t(x)).
\]
Inputs with the same pattern agree on every selected function.
Each pattern $\sigma\in\{0,1\}^t$ defines a set
\[
  A_\sigma(\mathbf g)=\{x\in\X:Z_{\mathbf g}(x)=\sigma\}.
\]
The nonempty sets $A_\sigma(\mathbf g)$ form the partition $\Pcal(\mathbf g)$ \emph{generated} by these functions.
We write $\minerr_{\mathbf g}(f)=\minerr_{\Pcal(\mathbf g)}(f)$ for the error of majority decoding from the selected labels.
A representation is \emph{generated by} $\mathbf g$ if its encoder depends on the input only through $Z_{\mathbf g}(x)$.

Writing out the full label vector takes $t$ bits, but many of the $2^t$ possible binary patterns may never occur on $\X$.
By definition of the dual growth function, at most $\Gdual(t)$ patterns occur.
We can therefore preserve all selected labels by giving each realized pattern a distinct index and storing that index in at most $\lceil\log\Gdual(t)\rceil$ bits.
The decoder can recover the pattern from the index and predict any task by its majority label among the inputs with that pattern.

This construction distinguishes the number of selected functions from the number of bits needed to store their labels.
The entropy-greedy procedure in \Cref{sec:greedy} selects functions until $\minerr_{\mathbf g}(f)\le\eps$ for every $f\in\F$, while controlling both quantities.
Keeping the actual function $\Gdual$ can give sharper bounds than replacing it by the dimension-only estimate \eqref{eq:sauer}; see \Cref{thm:growth-upper,prop:growth-lower}.

\FloatBarrier

\section{A short representation for every function}
\label{sec:tight}
\suppressfloats[t]

We prove the deterministic upper bound in \Cref{thm:main}, then give self-contained lower-bound arguments, including the known bounds that establish its worst-case optimality.
The construction starts with no information about the input and adds a function whenever its value cannot yet be predicted accurately.
Entropy measures progress, while the dual growth function limits how much information the selected functions can reveal.

\subsection{The entropy-greedy construction}
\label{sec:greedy}

Recall that $\minerr_{\mathbf g}(f)$ is the error of majority decoding for $f$ given the labels of the selected functions $\mathbf g$.
Adding a function refines the partition, and therefore cannot increase this error for any other function.
\Cref{alg:greedy} keeps refining until every function has error at most $\eps$.

\begin{algorithm2e}[t]
  \caption{\textsc{EntropyGreedy}$(\F,D,\eps)$}
  \label{alg:greedy}
  \LinesNumbered
  \DontPrintSemicolon
  $\mathbf g\gets()$\tcp*{selected functions}
  \While{there is $f\in\F$ with $\minerr_{\mathbf g}(f)>\eps$}{
    Append one such $f$ to $\mathbf g$\;
  }
  $\Enc(x)\gets$ index of $(g(x))_{g\in\mathbf g}$ among realized patterns\;
  $\Dec(z,f)\gets$ majority label of $f$ on cell $z$ under $D$\;
  \Return{$(\mathbf g,\Enc,\Dec)$}\;
\end{algorithm2e}

\paragraph{One pass through a finite class.}
For finite $\F$, scan its functions once in any fixed order and select a function exactly when its current majority error exceeds $\eps$.
A skipped function keeps error at most $\eps$ under every later refinement, and a selected function has zero error thereafter.
Thus one pass suffices, using $|\F|$ majority-error evaluations; repeated scans, which could require $O(|\F|^2)$ evaluations, are unnecessary.
This counts error evaluations, not total running time: integrating under $D$ and maintaining the partition may be costly.
For infinite classes, finding a violating function is an existence argument unless a suitable oracle is available.

\begin{theorem}[Deterministic upper bound]
  \label{thm:det-upper}
  For every $D$ and $0<\eps\le1/4$, \Cref{alg:greedy} selects fewer than $(4d^{*}/\eps)\log(1/\eps)$ functions and returns a universal representation with error at most $\eps$ and length at most
  \[
    \left\lceil d^{*}\log\left(\frac{4e}{\eps}\log\frac1\eps\right)\right\rceil
    \le4d^{*}\log\frac1\eps .
  \]
\end{theorem}

The selected label vector can be much longer than the representation.
The encoder stores only its index among the patterns that actually occur, preserving all selected labels while using fewer bits.
The following observation explains why the number of selections is also controlled.

\begin{lemma}[Entropy increment]
  \label{lem:entropy-increment}
  For $X\sim D$, a finite list $\mathbf g$ from $\F$, and $f\in\F$,
  \[
    H\bigl(Z_{(\mathbf g,f)}(X)\bigr)-H\bigl(Z_{\mathbf g}(X)\bigr)
    =H\bigl(f(X)\mid Z_{\mathbf g}(X)\bigr)\ge2\minerr_{\mathbf g}(f).
  \]
\end{lemma}

\begin{proof}
  The equality is the entropy chain rule.
  On a cell $A$ with positive mass, write $p_A=\Pr[f(X)=1\mid X\in A]$.
  Its contributions to conditional entropy and majority error are $D(A)h_2(p_A)$ and $D(A)\min\{p_A,1-p_A\}$, respectively.
  Sum the inequality $h_2(p_A)\ge2\min\{p_A,1-p_A\}$ over the cells.
\end{proof}

\begin{proof}[Proof of \Cref{thm:det-upper}]
  After any positive number $t$ of selections, each increment is greater than $2\eps$, whereas the selected functions realize at most $\Gdual(t)$ patterns.
  Hence
  \begin{equation}
    \label{eq:entropy-stopping}
    2\eps t<H(Z_{\mathbf g}(X))\le\log\Gdual(t)
    \le d^{*}\log(et/d^{*})\qquad(t\ge d^{*}).
  \end{equation}
  The left side grows linearly with $t$ and the right side only logarithmically.
  More precisely, writing $L=\log(1/\eps)$, this inequality forces $t<4d^{*}L/\eps$; \Cref{lem:solve-stopping} gives the elementary calculation.
  The same bound is immediate for $t<d^{*}$, so the procedure must stop, with every function's error at most $\eps$.
  The final partition has at most $(4eL/\eps)^{d^{*}}$ cells: apply Sauer--Shelah when $t\ge d^{*}$, and use $2^t\le2^{d^{*}}$ otherwise.
  Its cell index gives the asserted representation.
  Finally, $\log(4eL/\eps)\le3.25L$ for $L\ge2$, and $3.25d^{*}L+1\le4d^{*}L$.
\end{proof}

Keeping the actual growth function in \eqref{eq:entropy-stopping} can sharpen the bound for a particular class; \Cref{thm:growth-upper} states this refinement.

\subsection{Lower bounds for each fixed class}
\label{sec:instance-lower}

A dual-shattered family encodes an arbitrary binary string of length $d^{*}$.
Recovering each bit with small error requires retaining most of the string's entropy.
This is the classical information argument behind the VC lower bound for one-way communication \citep{kremer1999randomized,molinaro2015amplification}.
We state a version that allows the inputs realizing all binary patterns to carry only a fraction of the probability mass.

\begin{proposition}[Lower bound from a shattered set]
  \label{prop:instance-lower}
  Choose inputs $x_\sigma$, $\sigma\in\{0,1\}^{d^{*}}$, realizing the patterns of $d^{*}$ dual-shattered functions.
  Suppose $D$ gives mass $p/2^{d^{*}}$ to each $x_\sigma$, with total mass $p\in(0,1]$ on these inputs.
  For $0\le\eps<p/2$,
  \[
    \CC^{\to}_{\det}(\F,D,\eps)
    \ge\bigl(1-h_2(\eps/p)\bigr)d^{*}.
  \]
\end{proposition}

\begin{proof}
  Condition on the displayed inputs, so that their pattern $Y$ is uniform on $\{0,1\}^{d^{*}}$, and set $Z=\Enc(X)$.
  Write $\widehat Y_i(Z)\in\{0,1\}$ for the decoder's prediction of coordinate $i$.
  For each coordinate, the error indicator $E_i=\1\{Y_i\ne\widehat Y_i(Z)\}$ has mean at most $\eps/p$.
  Given $Z$, the indicator determines $Y_i$, so
  \[
    H(Y\mid Z)\le\sum_{i=1}^{d^{*}}H(Y_i\mid Z)
    =\sum_{i=1}^{d^{*}}H(E_i\mid Z)
    \le d^{*}h_2(\eps/p).
  \]
  Thus a $b$-bit representation satisfies $b\ge H(Z)\ge H(Y)-H(Y\mid Z)\ge(1-h_2(\eps/p))d^{*}$.
\end{proof}

With $p=1$, this recovers the deterministic consequence of the classical dimension lower bound in \Cref{prop:intro-1}, with a single distribution independent of $\eps$.
The accuracy dependence cannot be improved for every class using only $d^{*}$.

\begin{example}[A class that needs only $d^{*}$ bits even at zero error]
  \label{ex:instance-saturation}
  The $r$ coordinate functions on $\{0,1\}^{r}$ have dual VC dimension $r$.
  Sending the input itself gives an exact representation of length $r$ under every distribution, ruling out an instance-wise $\Omega(d^{*}\log(1/\eps))$ lower bound based only on $d^{*}$.
\end{example}

\Cref{sec:exp-sae} explores this saturation behavior in reconstruction-trained SAEs on finite task families. Additional dependence on accuracy can instead be obtained from the class's full dual growth function.

\begin{proposition}[Dual-growth lower bound]
  \label{prop:growth-lower}
  For every integer $t\ge1$, there is a finitely supported distribution $D_t$ such that, simultaneously for $0\le\eps<1/2$,
  \[
    \CC^{\to}_{\det}(\F,D_t,\eps)\ge\log\Gdual(t)-t h_2(\eps).
  \]
  Consequently, $\sup_D\CC^{\to}_{\det}(\F,D,\eps)\ge\sup_{t\ge1}\{\log\Gdual(t)-t h_2(\eps)\}$.
\end{proposition}

\begin{proof}
  Choose $t$ functions attaining $M=\Gdual(t)$; the maximum exists because the number of patterns is an integer between $1$ and $2^t$.
  Put the uniform distribution on one input per pattern.
  The resulting label vector $Y$ has entropy $\log M$.
  For any $b$-bit representation $Z=\Enc(X)$, the preceding error-indicator argument gives $H(Y\mid Z)\le t h_2(\eps)$.
  Therefore $b\ge H(Y)-H(Y\mid Z)\ge\log M-t h_2(\eps)$.
  Taking the supremum over $t$ proves the last assertion.
\end{proof}

Taking $t=d^{*}$ recovers the shattered-set bound.
For the coordinate class above, $\Gdual(t)=2^{\min\{t,d^{*}\}}$, and optimizing over $t$ gives exactly $(1-h_2(\eps))d^{*}$.

\subsection{Known worst-case lower bounds: a product of thresholds}
\label{sec:threshold}

The $\Omega(d^{*}\log(1/\eps))$ worst-case lower bound follows from the packing characterization of \citet{molinaro2015amplification}.
Their greater-than example gives the one-dimensional $\Omega(\log(1/\eps))$ dependence; applying the characterization to a product of thresholds gives the dependence on $d^{*}$ as well.
We explain this consequence of their result, then give a self-contained entropy proof with explicit constants under the uniform input distribution.
These lower bounds serve as an optimality benchmark for the deterministic upper bound in \Cref{thm:main}.

Let $D$ be uniform on $\X=[N]^r$ and define
\begin{equation}
  \label{eq:threshold-class}
  \F_{r,N}=\{f_{j,a}:j\in[r],\ a\in[N-1]\},
  \qquad f_{j,a}(x)=\1\{x_j\le a\}.
\end{equation}

\begin{lemma}[Dual dimension of the threshold product]
  \label{lem:threshold-dual}
  For $r\ge1$ and $N\ge2$, the class $\F_{r,N}$ has dual VC dimension $r$.
\end{lemma}

\begin{proof}
  The functions $f_{1,1},\ldots,f_{r,1}$ realize all patterns on $\{1,2\}^r$.
  Conversely, two thresholds from the same coordinate are nested, so they cannot be dual-shattered.
  A dual-shattered set therefore contains at most one function per coordinate.
\end{proof}

To connect this construction to the packing characterization, let $Q$ be uniform on $\F_{r,N}$.
The distance between two rows of its evaluation matrix is
\[
  \rho_Q(x,x')=\Pr_{f\sim Q}[f(x)\ne f(x')]
  =\frac{\|x-x'\|_1}{r(N-1)}.
\]
An $\ell_1$ ball of radius $r$ contains at most $2^r\binom{2r}{r}\le8^r$ integer points, so a greedy packing of $[N]^r$ gives at least $(N/8)^r$ points at pairwise $\ell_1$ distance greater than $r$.
Their $\rho_Q$ distances are greater than $1/(N-1)$ and hence greater than $8\eps$ when $\eps\le1/(20N)$.
The packing lower bound of \citet{molinaro2015amplification} therefore gives $\Omega(r\log N)$ bits under some input distribution for sufficiently large $N$, even with shared randomness; \Cref{sec:public} states the packing bound in our notation.
The following argument gives the explicit bound $r\log N-2r$ for the uniform distribution $D$ fixed above.

\begin{theorem}[Threshold-product lower bound]
  \label{thm:threshold-lower}
  For $r\ge1$, $N\ge2$, and $0<\eps\le1/(20N)$, every universal representation for $(\F_{r,N},D)$ with error at most $\eps$ has length at least $r\log N-2r$.
\end{theorem}

\begin{proof}
  Set $Z=\Enc(X)$.
  For each coordinate $j$, collect the threshold predictions into
  \[
    y_j(z)=(\Dec(z,f_{j,a}))_{a\in[N-1]}.
  \]
  Write $\chi(u)=(\1\{u\le a\})_{a\in[N-1]}$ and choose $a_j(z)\in[N]$ minimizing $\|\chi(u)-y_j(z)\|_1$.
  Since $\|\chi(u)-\chi(v)\|_1=|u-v|$, the triangle inequality and this minimizing choice give
  \[
    \sum_{j=1}^r\E|X_j-a_j(Z)|
    \le2\sum_{j=1}^r\E\|\chi(X_j)-y_j(Z)\|_1
    \le2\eps r(N-1)\le r/10.
  \]
  Thus the decoded point is close to $X$ in expected total coordinate distance.
  An integer-valued variable $E$ satisfies $H(E)\le\log3+\E|E|$: this follows from nonnegativity of relative entropy against the probability mass function $q(k)=2^{-|k|}/3$ on $\mathbb Z$.
  Applying this inequality to $E_j=X_j-a_j(Z)$ gives
  \[
    H(X\mid Z)\le\sum_{j=1}^rH(X_j\mid Z)
    =\sum_{j=1}^rH(E_j\mid Z)
    \le\sum_{j=1}^rH(E_j)
    \le r\log3+r/10<2r.
  \]
  Because $H(X)=r\log N$, we obtain $b\ge H(X)-H(X\mid Z)\ge r\log N-2r$.
\end{proof}

The proof only uses error averaged uniformly over the threshold queries.
For a public-coin representation, averaging over the shared seed gives a seed whose error under the product distribution $D\times Q$ is at most $\eps$.
Fixing this seed and applying the same proof gives the same $r\log N-2r$ lower bound.
Thus the construction also witnesses the known public-coin lower bound in \Cref{cor:public-tight}.

\paragraph{Matching the deterministic upper bound.}
\label{sec:main-proof}
Taking $r=d^{*}$ and $N=\lfloor1/(20\eps)\rfloor$ recovers the known worst-case order in terms of dimension and accuracy.
For sufficiently small $\eps$, $N\ge2$, $\eps\le1/(20N)$, and $\log N\ge\log(1/\eps)-\log40$.
\Cref{thm:threshold-lower} then gives $d^{*}\log(1/\eps)-(\log40+2)d^{*}$ bits, which is $\Omega(d^{*}\log(1/\eps))$.
The deterministic upper bound in \Cref{thm:main}, proved in \Cref{thm:det-upper}, therefore matches this known worst-case order.

In \Cref{sec:exp-greedy}, we experimentally verify the greedy construction's predicted accuracy-size tradeoff on this family. \Cref{sec:exp-pilot} explores whether reconstruction-trained sparse autoencoders exhibit a similar tradeoff.
\Cref{app:tight} provides the zero-error characterization, tight interpolation for finite classes, and bounds on the number of functions needed to generate a representation.

\section{Learning universal representations from samples}
\label{sec:samples}

The distribution need not be known exactly.
An unlabeled sample suffices to build the representation and all its decoders, but our two constructions trade representation length against sample size.
Throughout this section, $d=\VCdim(\F)\ge1$, $d^{*}=\VCdim(\F^{*})\ge1$, $0<\eps\le1/4$, and $0<\delta<1$.
We assume that the uniform-convergence events and the learner's choices are measurable; this is automatic for countable classes with fixed tie-breaking.

\subsection{The shared-sample model}
\label{sec:shared-sample}

\begin{definition}[Shared-sample learner]
  \label{def:shared-sample}
  A learner receives $S\sim D^n$ and outputs an encoder $Z_S:\X\to\Zcal_S$ and a binary decoder $\dec_{S,f}$ for every $f\in\F$.
  It has length $b$, error $\eps$, and confidence $1-\delta$ if, for every $D$, with probability at least $1-\delta$ over $S$ and any internal randomness,
  \[
    \left\lceil\log|\Zcal_S|\right\rceil\le b,
    \qquad
    \sup_{f\in\F}\Pr_{X\sim D}[\dec_{S,f}(Z_S(X))\ne f(X)]\le\eps.
  \]
\end{definition}

The class $\F$ is known, so its functions can be evaluated on the unlabeled sample.
The sample is shared setup information; only the code for each fresh input counts toward the representation length.

\subsection{Two upper bounds}
\label{sec:sample-greedy}
\label{sec:sample-cover}

The first learner runs entropy-greedy on the empirical distribution with target $\eps/4$, then predicts each function by its empirical majority in each cell.

\begin{theorem}[Learning a representation of optimal length]
  \label{thm:sample-optimal}
  For an absolute constant $C$, \Cref{alg:sample-greedy} has error $\eps$, confidence $1-\delta$, and length $O(d^{*}\log(1/\eps))$ whenever
  \[
    n\ge
    \left(\frac{C}{\eps}\right)^{d^{*}+1}
    \left(\log\frac1\eps\right)^{d^{*}+2}
    +\frac{C}{\eps}\log\frac1\delta.
  \]
\end{theorem}

The difficulty is that both the cells and their majority labels depend on the sample.
We control all rules that apply an arbitrary Boolean decoder to a short list of functions from $\F$.
A growth-function count bounds the complexity of their error events, and a relative VC bound transfers the empirical error to $D$.
The number of possible cell labels accounts for the exponential dependence on $d^{*}$.

The second learner selects a small empirical cover $\F_0\subseteq\F$: every function is close on the sample to some member of $\F_0$.
The representation records the joint labels of $\F_0$, and each decoder returns the label of a nearby member.

\begin{theorem}[A sample-efficient cover representation]
  \label{thm:cover-protocol}
  For an absolute constant $C$, \Cref{alg:cover} has error $\eps$, confidence $1-\delta$, and length at most $Cdd^{*}\log(1/\eps)$ whenever
  \[
    n\ge\frac{C}{\eps}\left(d\log\frac1\eps+\log\frac1\delta\right).
  \]
\end{theorem}

Here the error events are just disagreements between pairs of functions in $\F$, whose VC dimension is $O(d)$.
Haussler's packing bound controls the cover size, and dual growth controls the number of joint label patterns.
The algorithms and proofs are in \Cref{app:samples}.

\subsection{A necessary logarithmic factor}
\label{sec:sample-lower}

Even thresholds, with $d=d^{*}=1$, require more than $O(1/\eps)$ samples if the number of representation values must be independent of the domain size.

\begin{theorem}[Threshold sample lower bound]
  \label{thm:sample-lower}
  There is an absolute constant $c>0$ such that, for every sufficiently small $\eps>0$ and every integer $M\ge1$, there is an integer $N$ for which the following holds.
  Every shared-sample learner for the thresholds $\F=\{\1\{x\le a\}:a\in[N]\}$ on $[N]$ that uses at most $M$ representation values, has error $\eps$, and has confidence $2/3$, needs
  \[
    n\ge\frac{c}{\eps}\log\frac1\eps.
  \]
\end{theorem}

To see the obstruction, hide a uniform distribution on $K=\lfloor1/(2\eps)\rfloor$ points of a much larger line.
An error on even one support point has mass greater than $\eps$, so every threshold must be decoded exactly at every support point.
A fixed representation with $M$ values can decode all thresholds exactly at only $M$ points.
An unseen support point is unlikely to be one of them when $N$ is large.
Seeing all $K$ support points takes $\Omega(K\log K)$ samples; \Cref{app:sample-lower} makes this argument precise.

For thresholds this rate is achievable: the empirical cover in \Cref{alg:cover} has at most $1+8/\eps$ members and partitions the line into at most $2+8/\eps$ intervals.
Thus $O((\log(1/\eps)+\log(1/\delta))/\eps)$ samples suffice for $\log(1/\eps)+O(1)$ bits (\Cref{cor:threshold-samples}).
The unresolved issue is the dimension dependence for general classes.

\begin{question}[Sample complexity at optimal length]
  \label{q:sample}
  Can every class of primal and dual VC dimensions $d,d^{*}$ be learned to length $Cd^{*}\log(1/\eps)$ using
  $\mathrm{poly}(d,d^{*})\eps^{-c}\mathrm{polylog}(1/\eps,1/\delta)$ samples, where $C$ and $c$ are absolute constants?
\end{question}

\section{Extensions and implications}
\label{sec:implications}

The deterministic bound also applies beyond binary prediction.
For $\F\subseteq[0,1]^{\X}$, we allow decoder outputs in $[0,1]$ and measure expected absolute error.
Thresholding the functions reduces the problem to a binary class whose dual VC dimension is the pseudo-dimension $\Pdim(\F^{*})$ of the real-valued dual class.
Decoding by conditional medians then gives $O(\Pdim(\F^{*})\log(1/\eps))$ bits under absolute loss (\Cref{thm:real-upper}).

Randomness changes the class-specific picture.
For finite classes and domains, the product-distribution results of \citet{molinaro2015amplification} characterize the required length, up to a constant change in accuracy, by the largest covering entropy of $\F^{*}$ in $L_1(Q)$ over distributions $Q$ on tasks.
\Cref{sec:public} states this characterization in our notation and gives its cover and packing proofs.
The packing lower bound also applies to deterministic representations.
An analogous characterization for a single deterministic partition remains open.

The construction has a connection to reconstruction objectives, but that connection has limits.
For finite $\X$, optimal log-loss reconstruction from a deterministic code $Z$ costs $H(X\mid Z)=H(X)-H(Z)$ \citep{cover2006elements}.
Thus each greedy step lowers this loss.
Moreover, maximizing joint entropy over sufficiently many selected functions from $\F$ produces a universal representation (\Cref{thm:one-shot}).
The restriction to features from $\F$ matters: an arbitrary entropy-maximizing code can retain information unrelated to the desired tasks.
This leaves open when reconstruction yields compact representations that preserve every task. \Cref{sec:exp-sae,sec:exp-pilot} explore this question using SAEs on synthetic task families.

Finally, the approximation guarantee has an exact connection to multicalibration \citep{hebertjohnson2018multicalibration,gopalan2022omnipredictors}.
For every relabeling $Y$ of the inputs, consider the specific predictor $\E[Y\mid Z]$.
We measure multicalibration error by averaging over prediction values and weighting each group by its probability.
The worst mass-weighted $L_1$ multicalibration error with respect to a binary task $f$ is $\E\Var(f(X)\mid Z)$ when maximized over relabelings.
This quantity lies between one half and one times the majority error of $f$ on $Z$ (\Cref{thm:umc,cor:umc-equivalence}).
We also explain how ordinary calibration constrains alternative readouts in \Cref{sec:multicalibration}.

\section{Experimental verification and exploration}
\label{sec:experiments}
Our experiments explore whether trained sparse autoencoders (SAEs) exhibit the $d^* \log(1/\varepsilon)$ tradeoff of \cref{thm:det-upper,thm:threshold-lower}, and provide a setting to connect the theory to learned representations. First, we run the greedy procedure on vectors of independently sampled numbers; each task asks whether one coordinate is at most a chosen threshold. We call this collection of tasks the \textit{threshold product family} (\Cref{sec:exp-greedy}). Second, we study SAEs on finite task families (\Cref{sec:exp-sae}). Finally, we run the same type of autoencoders on a threshold product family (\Cref{sec:exp-pilot}). Implementation details are provided in \Cref{app:experiments}.

\subsection{Threshold-product family}
\label{sec:exp-greedy}
We test a greedy variant of \Cref{alg:greedy} on the threshold-product family in \eqref{eq:threshold-class}. At each step, we add the threshold with the largest majority error under the current partition. Inputs are uniform on $[N]^{d^{*}}$, with $N=256,128,64,32$ for $d^{*}=1,2,3,4$, respectively. We run the algorithm once per grid and record the first partition that meets each error target. The eleven targets range from $\eps=0.15$ down to $\eps=0.004$, using only those with $\eps\ge1/N$.
This experiment illustrates how the greedy construction's representation size grows as the error target decreases.
A threshold inside an interval of mass $m$ has majority error at most $m/2$, with equality at its middle. So each step halves a longest interval of some coordinate, the procedure stops once every interval has mass at most $2\eps$. Since every $N$ is a power of two the stored index has length
\begin{equation}
  \label{eq:greedy-exact}
  \log|\Pcal|=d^{*}\left\lceil\log\frac{1}{2\eps}\right\rceil
\end{equation}
for every grid and target in our range (\Cref{fig:greedy-scaling}), while the number of selected thresholds grows like $d^{*}/\eps$.
Not every class behaves this way: for $r$ independent fair binary coordinates and the coordinate functions, the procedure selects all $r$ coordinates for every target below one half, while the input itself provides zero error with $r$ bits (\Cref{ex:instance-saturation}).

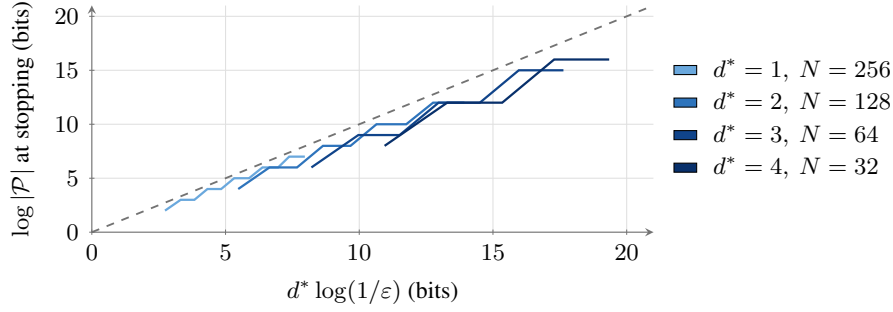
\begin{figure}[t]
  \centering
  \begin{tikzpicture}
    \begin{axis}[
      experiment axis,
      width=0.45\textwidth,height=3cm,
      xmin=0,xmax=21,ymin=0,ymax=21,
      xtick={0,5,10,15,20},ytick={0,5,10,15,20},
      xlabel={$d^{*}\log(1/\varepsilon)$ (bits)},
      ylabel={$\log|\mathcal P|$ at stopping (bits)},
      legend style={at={(1.03,0.5)},anchor=west,column sep=4pt},
    ]
      \addplot[experiment guide,forget plot] coordinates {(0,0) (21,21)};
      \addplot[experiment series,blueSix,no marks,forget plot]
        table[col sep=comma,x=predictor,y=bits]{\arxivroot figures/data/greedy_threshold_d1.csv};
      \addlegendimage{swatch legend,draw=none,fill=blueSix}
      \addlegendentry{$d^{*}=1,\ N=256$}
      \addplot[experiment series,blueFour,no marks,forget plot]
        table[col sep=comma,x=predictor,y=bits]{\arxivroot figures/data/greedy_threshold_d2.csv};
      \addlegendimage{swatch legend,draw=none,fill=blueFour}
      \addlegendentry{$d^{*}=2,\ N=128$}
      \addplot[experiment series,blueTwo,no marks,forget plot]
        table[col sep=comma,x=predictor,y=bits]{\arxivroot figures/data/greedy_threshold_d3.csv};
      \addlegendimage{swatch legend,draw=none,fill=blueTwo}
      \addlegendentry{$d^{*}=3,\ N=64$}
      \addplot[experiment series,blueOne,no marks,forget plot]
        table[col sep=comma,x=predictor,y=bits]{\arxivroot figures/data/greedy_threshold_d4.csv};
      \addlegendimage{swatch legend,draw=none,fill=blueOne}
      \addlegendentry{$d^{*}=4,\ N=32$}
    \end{axis}
  \end{tikzpicture}
  \caption{Representation size of the maximum-error greedy construction on threshold products, plotted against $d^{*}\log(1/\varepsilon)$. The curves follow the expression in \cref{eq:greedy-exact}; the dashed line is $y=x$.}
  \label{fig:greedy-scaling}
\end{figure}

\subsection{Finite-family capacity saturation}
\label{sec:exp-sae}
We ask how many bits a sparse autoencoder, trained only for reconstruction, requires to support each function in a finite family at a target error. We then decrease the target error to observe whether the required bits increase and whether observation noise is a major source of error.

\paragraph{Setup.}
Each family is a set of binary functions on a finite set of equally likely inputs, which we call \emph{atoms}: the \emph{powerset} family contains every labeling of $n_{\mathrm a}\in\{2,3,4,6,8,10,12,16\}$ atoms ($d^{*}=\lfloor\log n_{\mathrm a}\rfloor$), and the \emph{parity} and \emph{coordinate} families contain the functions $x\mapsto w^{\top}x\bmod2$ and the $r$ coordinate functions on $\{0,1\}^{r}$, where $d^{*}=r$, for $1\le r\le6$ and $2\le r\le6$.
The encoder never sees an atom: each atom $A$ has a fixed prototype $\mu_A\in\R^{64}$, the encoder sees $V=\mu_A+\xi$ with fresh Gaussian noise $\xi$, the label of a task $f$ is $f(A)$, and all splits share the atoms and functions.
We train top-$k$ sparse autoencoders \citep{makhzani2013ksparse,gao2024scaling} with $W$ dictionary features, keeping the $k$ largest activations for each input. After training, we round each retained activation to one of $2^{q}$ levels including zero. We define the \emph{representation bit budget} as
\[
  B_{\mathrm{bits}}(W,k,q)=\left\lceil\log\sum_{i=0}^{k}\binom{W}{i}(2^{q}-1)^{i}\right\rceil.
\]
This budget counts the bits needed for a fixed-length index of all allowed codes, including codes the trained encoder never uses. We test 36 configurations on each instance, with five random seeds per configuration.
For each task, a lookup decoder predicts the most common training label among inputs with the same code. We measure worst-case error as the largest prediction error across all tasks in the family. For each instance, random seed, and error target $\varepsilon$, we select the tested configuration with the smallest representation bit budget whose validation error is at most $\varepsilon$.

We then evaluate the selected configuration on a test set and call it \textit{certified} if its test error also meets the target. Both checks use a one-sided 95\% upper confidence threshold on the worst-case error.
We compare the minimum bit budgets with two theoretical references in the noiseless setting, where the encoder observes the atom directly. The lossless reference $B_0=\lceil\log(\text{number of atoms})\rceil\le d^{*}+1$ is the number of bits needed to identify the atom, allowing every task to be answered without error. The shattered-set bound $B_{\mathrm{sh}}(\eps)=(1-h_2(\eps/p_{\mathrm{sh}}))d^{*}$, where $p_{\mathrm{sh}}$ is the mass of the dually shattered inputs ($2^{d^{*}}/n_{\mathrm a}$ for the powerset family and $1$ for the other two), provides a lower bound on the bits required to meet the target error (\Cref{prop:instance-lower}).

\paragraph{Results.}
The minimum bit budget increases only slightly as the error target decreases. In $\StudySaturationInstances$ of the $\StudyEndpointInstances$ instances with results at multiple targets, the median budget increases by at most two bits between the loosest and strictest selected targets (\Cref{fig:capacity-target}). This limited growth is consistent with the noiseless reference: these finite families admit small representations that already preserve every task without error.

The validation-selected representations also retain their accuracy on the test set. All $\StudySelectedOutcomes$ selections meet their targets in observed test error, and $\StudyCertifiedOutcomes$ also pass the confidence-threshold check. In $\StudyOutcomesNoCandidate$ other cases, none of the configurations passes the validation check. The median budget is typically one to three bits above the lossless reference $B_0$, and it exceeds the shattered-set lower bound in every setting with a selection (\Cref{fig:sae}). The prototype-based comparison suggests that observation noise alone does not explain the extra bits: an encoder that assigns each noisy observation to its nearest known prototype can meet error targets by using only $B_0$ bits.

\begin{figure}[t]
  \centering
  \newcommand{\PlotCapacity}[4]{\addplot[experiment series,#2,no marks,forget plot]
      table[col sep=comma,x=log2_inv_epsilon,y=median_bits,
        discard if not={condition}{#3}]
        {\arxivroot figures/data/capacity_#1.csv};
    \addplot[#2,dotted,line width=0.6pt,no marks,forget plot]
      table[col sep=comma,x=log2_inv_epsilon,y=zero_error_bits,
        discard if not={condition}{#3}]
        {\arxivroot figures/data/capacity_#1.csv};
    \addplot[#2,only marks,mark=triangle,mark options={rotate=180},mark size=1.6pt,forget plot]
      table[col sep=comma,x=log2_inv_epsilon,y expr=13.3,
        discard if not={condition}{#3},discard if not={any_selected}{0}]
        {\arxivroot figures/data/capacity_#1.csv};
    \addlegendimage{swatch legend,draw=none,fill=#2}
    \addlegendentry{#4}
  }
  \begin{tikzpicture}
    \begin{groupplot}[
      group style={group size=3 by 1,horizontal sep=1.2cm},
      experiment axis,height=2.8cm,
      xmin=2,xmax=8,ymin=0,ymax=14,
      xtick={2,4,6,8},ytick={0,4,8,12},
      unbounded coords=discard,
      xlabel={$\log(1/\varepsilon)$},
      legend columns=2,
      legend style={at={(0.5,-0.5)},anchor=north,font=\tiny,column sep=4pt},
    ]
      \nextgroupplot[title={(a) Parity},ylabel={Minimum bits}]
        \PlotCapacity{parity}{blueOne}{parity-1}{$r=1$}
        \PlotCapacity{parity}{blueTwo}{parity-2}{$r=2$}
        \PlotCapacity{parity}{blueThree}{parity-3}{$r=3$}
        \PlotCapacity{parity}{blueFour}{parity-4}{$r=4$}
        \PlotCapacity{parity}{blueFive}{parity-5}{$r=5$}
        \PlotCapacity{parity}{blueSix}{parity-6}{$r=6$}
      \nextgroupplot[title={(b) Powerset}]
        \PlotCapacity{powerset}{blueOne}{powerset-2}{$n_{\mathrm a}=2$}
        \PlotCapacity{powerset}{blueTwo}{powerset-3}{$n_{\mathrm a}=3$}
        \PlotCapacity{powerset}{blueThree}{powerset-4}{$n_{\mathrm a}=4$}
        \PlotCapacity{powerset}{blueFour}{powerset-6}{$n_{\mathrm a}=6$}
        \PlotCapacity{powerset}{blueFive}{powerset-8}{$n_{\mathrm a}=8$}
        \PlotCapacity{powerset}{blueSix}{powerset-10}{$n_{\mathrm a}=10$}
        \PlotCapacity{powerset}{blueSeven}{powerset-12}{$n_{\mathrm a}=12$}
        \PlotCapacity{powerset}{blueEight}{powerset-16}{$n_{\mathrm a}=16$}
      \nextgroupplot[title={(c) Coordinates}]
        \PlotCapacity{transposed_powerset}{blueTwo}{transposed_powerset-2}{$r=2$}
        \PlotCapacity{transposed_powerset}{blueThree}{transposed_powerset-3}{$r=3$}
        \PlotCapacity{transposed_powerset}{blueFour}{transposed_powerset-4}{$r=4$}
        \PlotCapacity{transposed_powerset}{blueFive}{transposed_powerset-5}{$r=5$}
        \PlotCapacity{transposed_powerset}{blueSix}{transposed_powerset-6}{$r=6$}
    \end{groupplot}
  \end{tikzpicture}
  \caption{Minimum tested bit budget passing validation at each error target, shown as the median over seeds with a configuration. Dotted lines indicate the lossless reference $B_0$. Open triangles indicate instances for which no configuration passes validation.}
  \label{fig:capacity-target}
\end{figure}
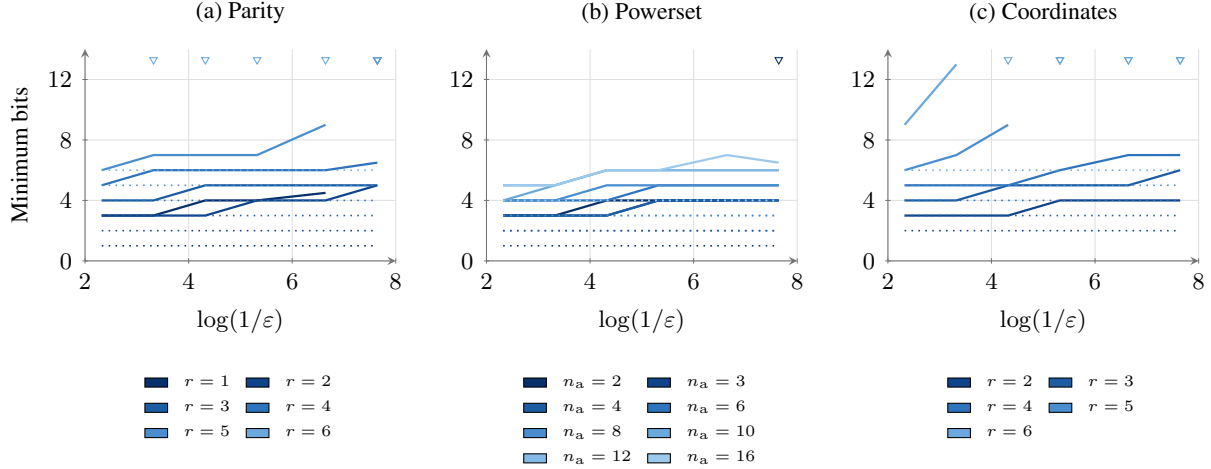

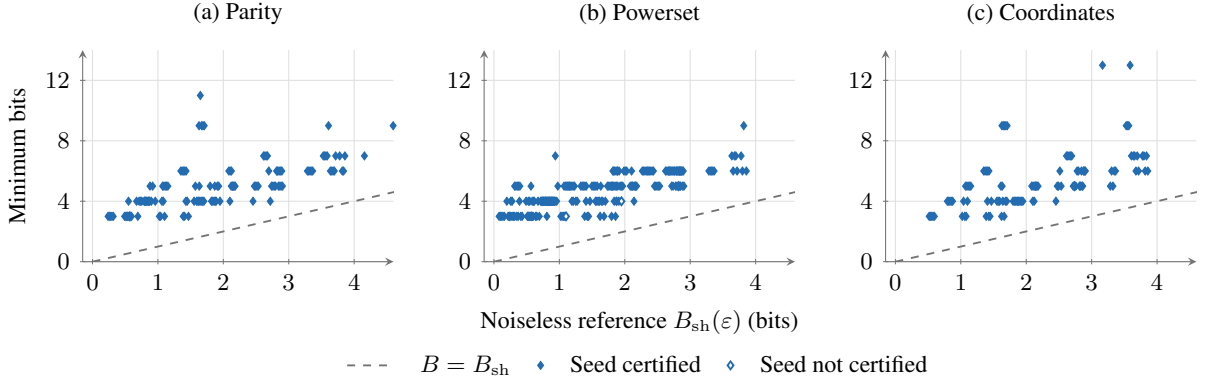
\begin{figure}[t]
  \centering
  \newcommand{\PlotTheorySAE}[1]{\addplot[experiment guide,forget plot]
      coordinates {(0,0) (5,5)};
    \addplot[only marks,experimentBlue,mark size=1.5pt,forget plot,
      scatter,scatter src=explicit symbolic,
      scatter/classes={
        1={mark=diamond*,fill=experimentBlue},
        0={mark=diamond*,fill=white,line width=0.7pt}}]
      table[col sep=comma,
        x expr=\thisrow{lower_bits}+(\thisrow{seed}-7)*0.02,y=selected_bits,
        meta=heldout_pass]
        {\arxivroot figures/data/seeds_#1.csv};
  }
  \begin{tikzpicture}
    \begin{groupplot}[
      group style={group size=3 by 1,horizontal sep=1.2cm},
      experiment axis,height=2.8cm,
      xmin=-0.15,xmax=4.6,ymin=0,ymax=14,
      xtick={0,1,2,3,4},ytick={0,4,8,12},
    ]
      \nextgroupplot[title={(a) Parity},ylabel={Minimum bits}]
        \PlotTheorySAE{parity}
      \nextgroupplot[title={(b) Powerset},xlabel={Noiseless reference $B_{\mathrm{sh}}(\varepsilon)$ (bits)}]
        \PlotTheorySAE{powerset}
      \nextgroupplot[title={(c) Coordinates},
        legend to name=legend:theory-frontier,legend columns=3,
        legend style={column sep=8pt}]
        \PlotTheorySAE{transposed_powerset}
        \addlegendimage{experiment guide}
        \addlegendentry{$B=B_{\mathrm{sh}}$}
        \addlegendimage{only marks,experimentBlue,mark=diamond*,
          mark options={fill=experimentBlue},mark size=1.5pt}
        \addlegendentry{Seed certified}
        \addlegendimage{only marks,experimentBlue,mark=diamond*,
          mark options={fill=white,line width=0.7pt},mark size=1.5pt}
        \addlegendentry{Seed not certified}
    \end{groupplot}
    \node[anchor=north,font=\footnotesize]
      at ([yshift=-1.0cm]group c2r1.south)
      {\pgfplotslegendfromname{legend:theory-frontier}};
  \end{tikzpicture}
  \caption{Bit budgets against the shattered-set bound $B_{\text{sh}}(\varepsilon)$. Each point represents one selected seed, with offsets for visibility. Filled diamonds were certified on test. The dashed line indicates $B_{\text{bits}}=B_{\text{sh}}(\varepsilon)$.}
  \label{fig:sae}
\end{figure}

\subsection{Accuracy-bit tradeoffs on threshold products}
\label{sec:exp-pilot}
The finite families in \Cref{sec:exp-sae} admit short lossless representations. To examine how prediction accuracy changes with representation size, we consider the threshold-product family ($\mathcal F_{r,1024}$) from \Cref{sec:threshold}. Each input contains \(r\in\{1,2,4,8,10\}\) independent coordinates, sampled uniformly from $[1024]$, rescaled to $[0,1]$. The task family contains all coordinate thresholds and has dual VC dimension ($d^*=r$). Representing the input exactly requires ($10d^*$) bits.

We train sparse autoencoders for reconstruction and vary the quantization precision of their activations to obtain different representation budgets. We evaluate each representation by predicting threshold labels from the reconstructed coordinates and measuring the largest error across all thresholds.

\Cref{fig:pilot} shows that prediction error decreases as the number of bits per dimension increases, with a similar trend across the five dimensions tested. The observed error decreases approximately exponentially with the normalized bit budget ($B_{\mathrm{bits}}/d^*$). For comparison, equal-width quantization of each coordinate also achieves exponentially decreasing worst-threshold error. These results illustrate the accuracy-bit tradeoff for learned representations on the threshold-product family.

\begin{figure}[t]
  \centering
  \newcommand{\PlotPilotSeeds}[3]{\addplot[only marks,#2,mark=diamond*,mark size=1.7pt,forget plot]
      table[col sep=comma,x=bits_per_dimension,y=worst_error,
        discard if not={series}{reconstruction-r#1-test}]
        {\arxivroot figures/data/threshold_pilot_points.csv};
    \addlegendimage{swatch legend,draw=none,fill=#2}
    \addlegendentry{#3}
  }
  \begin{tikzpicture}
    \begin{axis}[
      experiment axis,
      width=0.56\textwidth,height=3.5cm,
      ymode=log,log basis y=2,
      xmin=1,xmax=7,ymin=0.005,ymax=1,
      xtick={1,2,3,4,5,6,7},
      unbounded coords=discard,
      xlabel={$B_{\mathrm{bits}}/d^*$ (bits per dimension)},
      ylabel={Worst-case test error},
      legend style={at={(1.05,0.5)},anchor=west,row sep=3pt,column sep=4pt},
      legend image post style={mark size=2.4pt},
    ]
      \addplot[only marks,black!80,line width=0.8pt,mark=x,mark size=2.4pt,domain=1:6,samples=6] {2^(-x-1)};
      \addlegendentry{Equal-width quantization}
      \PlotPilotSeeds{1}{blueEight}{$d^{*}=1$}
      \PlotPilotSeeds{2}{blueSeven}{$d^{*}=2$}
      \PlotPilotSeeds{4}{blueFive}{$d^{*}=4$}
      \PlotPilotSeeds{8}{blueThree}{$d^{*}=8$}
      \PlotPilotSeeds{10}{blueOne}{$d^{*}=10$}
    \end{axis}
  \end{tikzpicture}
  \caption{Worst-case error against the bit budget per dimension for reconstruction-trained SAEs, one dot per random seed (three per dimension) for each $d^{*}\in\{1,2,4,8,10\}$.
    $\times$ are the exact error $2^{-\ell-1}$ of equal-width quantization using $\ell$ bits per coordinate.
    Prediction error decreases as the representation budget increases.}
  \label{fig:pilot}
\end{figure}
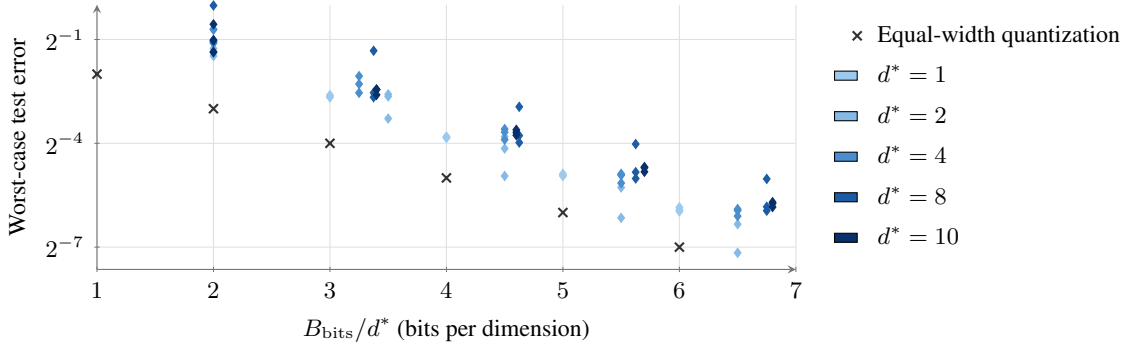

\section{Discussion}
\label{sec:discussion}

A universal representation must retain enough information to answer every task, yet it can be much shorter than the list of task labels used to construct it.
The proof makes this distinction quantitative: each useful refinement gains entropy, while dual VC dimension limits the number of possible joint patterns.
This gives a deterministic representation of optimal worst-case length and a one-pass implementation for finite classes.
The coordinate example shows why the same logarithmic lower bound cannot hold for every class.

Three questions are especially natural.
First, can a deterministic partition match the class-specific length achievable with shared randomness?
The dual-growth and packing lower bounds identify useful obstructions but do not settle this comparison.
Second, can one learn a representation of length $O(d^{*}\log(1/\eps))$ with a sample bound whose exponent in $1/\eps$ is independent of $d^{*}$?
The threshold case is understood, but our general learner pays for arbitrary readouts on an adaptively chosen partition.
Third, the number of selected functions is known only between $\Omega(d^{*}/\eps)$ and $O((d^{*}/\eps)\log(1/\eps))$ in the worst case.
Removing or proving necessary this remaining logarithm would sharpen the feature-selection guarantee.

The representation length is an information measure.
It does not bound the computation needed to evaluate the selected functions, find them in an implicit class, or fit every decoder from limited data.
Likewise, a short code with unrestricted readouts need not give a low-dimensional representation with linear readouts.
Understanding when additional structure makes these representations computationally and statistically practical is a separate question.

\section*{AI Disclosure}
OpenAI's Codex was used to revise the exposition, assist with mathematical arguments and proof checks, inspect references, and prepare the LaTeX source.
In particular, it helped develop the entropy proofs of the upper bounds and identify scope and quantifier issues in earlier formulations. Anthropic's Claude Code was used to help prepare figures and assist with writing experimental code.

\bibliography{\arxivroot references}
\appendix
\crefalias{section}{appendix}
\section{Additional Related Work}
\label{app:additional-related}

We complement the direct comparisons in \Cref{sec:related} with connections to other learning and representation models.

\paragraph{Evaluation problems and private learning.}
\citet{feldman2015sample} relate private learning and representation dimension to the randomized one-way complexity of evaluating a concept.
Their communication direction is the reverse of ours: the concept holder sends a message to the example holder.
In that orientation the row class is $\F$, with dimension $d$, rather than $\F^{*}$.
They distinguish deterministic and probabilistic families of hypotheses used to approximate a target, and relate these to private-coin and public-coin protocols, respectively.
In particular, their term \emph{deterministic representation dimension} does not mean deterministic communication.
These connections establish the importance of evaluation problems in learning theory, but their direction, privacy requirement, and approximation objects differ from the deterministic compression of an input considered here.

For a fixed input distribution, approximation by covers in the disagreement metric also appears in private learning with a known reference distribution \citep{chaudhuri2011sample}.
\citet{feldman2015sample} identify fixed-distribution deterministic representation dimension with the logarithm of the smallest such cover, allowing hypotheses outside the original class.

\paragraph{Disambiguation and sample compression.}
Dual VC dimension also enters disambiguation of partial concept classes: \citet{alon2021theory} use minimax and uniform convergence on the dual class to construct small majority representations.
That construction completes partially specified hypotheses, whereas our encoder partitions the input domain.
Likewise, the embedding obstructions of \citet{chase2024dual} concern compression of labeled samples rather than a finite code for each input.

\paragraph{Other communication and representation models.}
\citet{kane2019communication} study classification from samples distributed between communicating parties, counting transmitted examples as well as bits.
Their coVC dimension measures the size of certificates of nonrealizability and is distinct from dual VC dimension.
For representations with linear readouts, sign rank and its approximate variants provide different complexity measures \citep{alon2016sign,kamath2020approximate}.
For example, bounded VC dimension alone does not bound sign rank \citep{alon2016sign}.
Our bit-length bounds allow arbitrary task-specific decoders and therefore do not imply comparably small dimension or norm for linear representations.

\paragraph{Representations shared across tasks.}
\citet{baxter2000model}, \citet{maurer2016benefit}, and \citet{tripuraneni2020theory} study how a representation learned from several related tasks reduces the sample complexity of new tasks.
They measure the complexity of a class of representation maps and of the predictors placed on top of them, and they assume that the tasks share a low-dimensional structure.
In the task reading of our problem, we instead fix the class of tasks, allow arbitrary decoders, require one representation to serve every task in the class, and count bits.
Our results identify a combinatorial property of the class, the dual VC dimension, that determines how much any such shared representation must store.

\paragraph{Information-theoretic objectives.}
Rate-distortion theory \citep{shannon1959coding} asks for the fewest bits that describe a source within a given distortion, and the information bottleneck \citep{tishby1999information} trades the information a representation keeps about the input against the information it keeps about one target variable.
Our distortion is the worst error over a whole class of targets, and we measure the representation length deterministically rather than by mutual information.
The entropy of the representation nonetheless drives our construction, and \Cref{sec:objectives} makes the link to log-loss reconstruction explicit.

\paragraph{Multicalibration.}
Multicalibration \citep{hebertjohnson2018multicalibration} and omniprediction \citep{gopalan2022omnipredictors} ask for a single predictor that is simultaneously good for a class of tests or losses.
Universal representations ask for a single encoder that serves a whole class of functions.
\Cref{sec:multicalibration} shows that the two are equivalent up to a factor of two in the error, once one asks for multicalibration under every relabeling of the data.

\section{Interpretations of the representation guarantee}
\label{app:interpretations}

\begin{remark}[Boundary mass]
  \label{rem:boundary}
  For a measurable set $B$ and a finite partition $\Pcal$, let $\partial(B:\Pcal)$ be the union of the cells $A\in\Pcal$ with $D(A\cap B)>0$ and $D(A\setminus B)>0$.
  A cell outside $\partial(B:\Pcal)$ contributes nothing to $\minerr_{\Pcal}(\1_B)$, and a cell inside contributes at most half its mass.
  Hence $\minerr_{\Pcal}(\1_B)\le\frac12 D(\partial(B:\Pcal))$, and the theorem of \citet{adams2012uniform} implies that $\CC^{\to}_{\det}(\F,D,\eps)<\infty$ for every $\eps>0$ whenever $\F$ has finite VC dimension.
\end{remark}

\subsection{Agnostic prediction and other interpretations}
\label{sec:views}

By \Cref{prop:partition}, a universal representation of length $b$ with error $\eps$ is the same object as a partition of $\X$ into at most $2^{b}$ cells on which every member of $\F$ has majority error at most $\eps$.
This object can be read in several ways, and the results of this paper apply to all of them.

\paragraph{Downstream tasks.}
If the members of $\F$ are tasks that will be asked of a learned representation, then $\Enc$ is a task-agnostic representation and $\Dec(\cdot,f)$ is the best readout for task $f$.

\paragraph{Queries and one-way communication.}
If the members of $\F$ are queries, then $\Enc(x)$ is a sketch of the data point $x$ from which every query can be answered, and $\CC^{\to}_{\det}(\F,D,\eps)$ is the deterministic one-way distributional communication complexity of evaluating $f(x)$, where Alice holds $x\sim D$, Bob holds $f$, and the error is measured for the worst $f$ \citep{kushilevitz1997communication}.

\paragraph{Coarse graining.}
If the members of $\F$ are observables, we can regard the cells of the partition as macro-states.
For each observable, we predict its majority value within each macro-state and require the resulting error, averaged over macro-states according to their probabilities under $D$, to be at most $\eps$.
\Cref{rem:boundary} relates this to the finite approximability of \citet{adams2012uniform}.

\paragraph{Hypothesis classes.}
If $\F$ is a hypothesis class, then a universal representation is a feature map that is sufficient for agnostic learning with $\F$, simultaneously for every labeling of the inputs.

\begin{proposition}[Agnostic sufficiency]
  \label{prop:agnostic}
  For an encoder $\Enc:\X\to\{0,1\}^{b}$ and $\eps\ge0$, the following are equivalent.
  \begin{enumerate}[label=(\roman*),leftmargin=2em]
    \item With majority decoding, $\Enc$ is a universal representation of $\F$ with error at most $\eps$.
    \item For every distribution $P$ of a pair $(X,Y)$ with $X\sim D$ and $Y\in\{0,1\}$,
          \[
            \min_{\beta:\{0,1\}^{b}\to\{0,1\}}\Pr_P\bigl[\beta(\Enc(X))\neq Y\bigr]\le\inf_{f\in\F}\Pr_P\bigl[f(X)\neq Y\bigr]+\eps .
          \]
  \end{enumerate}
\end{proposition}

\begin{proof}
  Assume (i), fix $P$ and $f\in\F$, and let $\beta_f$ be the majority decoder of $f$.
  The disagreement bound gives
  \begin{align*}
    \Pr_P[\beta_f(\Enc(X))\ne Y]
    &\le\Pr_P[f(X)\ne Y]+\Pr_D[\beta_f(\Enc(X))\ne f(X)]\\
    &\le\Pr_P[f(X)\ne Y]+\eps.
  \end{align*}
  Taking the infimum over $f$ gives (ii).
  Conversely, given (ii) and $f\in\F$, take $Y=f(X)$: some $\beta$ has $\Pr_D[\beta(\Enc(X))\neq f(X)]\le\eps$, so the majority error of $f$ is at most $\eps$.
\end{proof}

So replacing the inputs by an $\eps$-universal representation costs at most $\eps$ in approximation error for every target, and the representation is computed without seeing any labels.

\paragraph{Calibrated predictors.}
\Cref{sec:multicalibration} gives a further equivalent reading, up to a factor of two in $\eps$: for every labeling, the conditional-mean predictor based on the representation is multicalibrated with respect to $\F$.

\section{Refinements of the deterministic bounds}
\label{app:tight}

\subsection{Details of the stopping calculation}

\begin{lemma}[Few realized patterns]
  \label{lem:few-patterns}
  For every list $\mathbf g$ of $t$ functions in $\F$, $H(Z_{\mathbf g}(X))\le\log\Gdual(t)$, and if $t\ge d^{*}$ then $\log\Gdual(t)\le d^{*}\log(et/d^{*})$.
\end{lemma}

\begin{proof}
  The first bound follows because $Z_{\mathbf g}(X)$ takes at most $\Gdual(t)$ values; the second is the Sauer--Shelah bound~\eqref{eq:sauer}.
\end{proof}

\begin{lemma}[Solving the stopping inequality]
  \label{lem:solve-stopping}
  Let $0<\eps\le1/4$ and $t\ge d^{*}$.
  Then
  \[
    2\eps t<d^{*}\log(et/d^{*})
    \quad\Longrightarrow\quad
    t<\frac{4d^{*}}{\eps}\log\frac1\eps.
  \]
\end{lemma}

\begin{proof}
  Let $y=2\eps t/d^{*}$ and $L=\log(1/\eps)\ge2$.
  The hypothesis reads $y<L+\log(ey/2)$.
  The function $y\mapsto y-\log(ey/2)$ is increasing for $y\ge1/\ln2$, and at $y=8L\ge16$ it equals $8L-\log(4eL)\ge L$, because $7L\ge\log(4e)+\log L$ for $L\ge2$.
  Hence $y<8L$, that is, $t<4d^{*}L/\eps$.
\end{proof}

\subsection{A bound from the full dual growth function}
The proof uses the dual VC dimension only through the Sauer--Shelah bound on $\Gdual$.
Keeping $\Gdual$ itself gives a bound that can be much stronger for classes whose dual growth is smaller than the Sauer--Shelah bound.

\begin{definition}[Stopping index]
  \label{def:stopping}
  For $0<\eps<1/2$, let $t^{\star}_{\F}(\eps)=\min\{t\ge1:\log\Gdual(t)\le2\eps t\}$, with $t^{\star}_{\F}(\eps)=\infty$ if no such $t$ exists.
\end{definition}

\begin{theorem}[Dual-growth upper bound]
  \label{thm:growth-upper}
  Let $D$ be a distribution on $\X$ and $0<\eps<1/2$.
  If $t^{\star}=t^{\star}_{\F}(\eps)<\infty$, then \Cref{alg:greedy} stops after selecting $t<t^{\star}$ functions, every $f\in\F$ then satisfies $\minerr_{\mathbf g}(f)\le\eps$, and
  \[
    \CC^{\to}_{\det}(\F,D,\eps)\le\bigl\lceil\log\Gdual(t^{\star}-1)\bigr\rceil .
  \]
\end{theorem}

\begin{proof}
  As in the proof of \Cref{thm:det-upper}, after any positive number $t$ of selections, $2\eps t<\log\Gdual(t)$.
  If the procedure made $t^{\star}$ selections, this would contradict the definition of $t^{\star}$, so it stops after some $t<t^{\star}$ selections, with at most $\Gdual(t)\le\Gdual(t^{\star}-1)$ cells.
\end{proof}

\subsection{Zero error and finite classes}
\label{sec:interpolation}

At zero error, the answer is the number of distinct label vectors on the support.
To avoid null sets, we state it for finitely supported distributions.

\begin{proposition}[Zero-error complexity]
  \label{prop:zero-error}
  If $D$ has finite support $S$ and $\Lambda_D=|\{(f(x))_{f\in\F}:x\in S\}|$, then $\CC^{\to}_{\det}(\F,D,0)=\lceil\log\Lambda_D\rceil$.
  If $\F$ is finite with $|\F|=m$, then $\Lambda_D\le\sum_{i=0}^{d^{*}}\binom{m}{i}$.
\end{proposition}

\begin{proof}
  Two support points with different label vectors are separated by some function, so a zero-error representation must put them in different cells; conversely, the index of the label vector is a zero-error representation.
  The bound on $\Lambda_D$ is the Sauer--Shelah lemma for $\F^{*}$, viewed as a class on the $m$ points of $\F$.
\end{proof}

For a finite class of size $m\ge d^{*}$ and $0<\eps\le1/4$, writing the index of the full label vector works for every distribution, so \Cref{thm:det-upper,prop:zero-error} give
\begin{equation}
  \label{eq:finite-upper}
  \CC^{\to}_{\det}(\F,D,\eps)\le\min\left\{4d^{*}\log\frac1\eps,\ \left\lceil d^{*}\log\frac{em}{d^{*}}\right\rceil\right\}.
\end{equation}

The upper bound \eqref{eq:finite-upper} interpolates between the approximate bound and the zero-error bound, and the threshold product shows that this interpolation is tight.

\begin{theorem}[Finite-class interpolation]
  \label{thm:interpolation}
  There are absolute constants $c,C,c_0>0$ such that for every $d^{*}\ge1$, every $m\ge Cd^{*}$, and every $0<\eps<c_0$, there are a class $\F$ with $|\F|\le m$ and $\VCdim(\F^{*})=d^{*}$ and a distribution $D$ such that
  \[
    \CC^{\to}_{\det}(\F,D,\eps)\ge c\,d^{*}\log\min\left\{\frac1\eps,\frac{m}{d^{*}}\right\}.
  \]
\end{theorem}

\begin{proof}
  Put $M=\min\{1/\eps,m/d^{*}\}$ and $N=\lfloor M/20\rfloor$.
  The class $\F_{d^{*},N}$ has $d^{*}(N-1)\le m$ functions and dual VC dimension $d^{*}$.
  Also $\eps\le1/(20N)$, so \Cref{thm:threshold-lower} gives $d^{*}(\log N-2)$ bits.
  Taking $C=1600$ and $c_0=1/1600$ ensures $M\ge1600$ and $N\ge M/40\ge40$.
  Therefore $\log N-2\ge\frac12\log N\ge\frac14\log M$, proving the claim with $c=1/4$.
\end{proof}

\subsection{How many functions must a proper representation use?}
\label{sec:feature-count}

The entropy-greedy representation is a function of the labels of the selected functions.
We now count the selected functions rather than the bits.

\begin{definition}[Selection complexity]
  \label{def:feature-complexity}
  The \emph{selection complexity} $\mathrm{sel}_\eps(\F,D)$ is the least $t$ such that there are $g_1,\ldots,g_t\in\F$ and, for every $f\in\F$, a map $\dec_f:\{0,1\}^{t}\to\{0,1\}$ with
  \[
    \Pr_{X\sim D}\bigl[\dec_f(g_1(X),\ldots,g_t(X))\neq f(X)\bigr]\le\eps .
  \]
\end{definition}

Every $f$ is thus nearly a Boolean function of the selected functions; this does not require $f$ to lie in their linear span.

\begin{corollary}[Selection upper bound]
  \label{cor:feature-upper}
  For every distribution $D$ and every $0<\eps\le1/4$,
  \[
    \mathrm{sel}_\eps(\F,D)<\frac{4d^{*}}{\eps}\log\frac1\eps .
  \]
\end{corollary}

\begin{proof}
  Use the functions selected by \Cref{alg:greedy}, which are fewer than $(4d^{*}/\eps)\log(1/\eps)$ by \Cref{thm:det-upper}, and let $\dec_f$ output the majority label of $f$ on each cell.
\end{proof}

\begin{theorem}[Selection lower bound]
  \label{thm:feature-lower}
  There are absolute constants $c,c_0>0$ such that for every $r\ge1$ and $0<\eps<c_0$, the class $\F_{r,N}$ with $N=\lfloor1/(4\eps)\rfloor$ and the uniform distribution satisfy $\mathrm{sel}_\eps(\F_{r,N},D)\ge c\,r/\eps$.
\end{theorem}

\begin{proof}
\label{app:feature-lower}
Let $N=\lfloor1/(4\eps)\rfloor$, so that $N\ge2$ when $\eps<1/8$, and let $D$ be uniform on $[N]^{r}$.
Suppose that $g_1,\ldots,g_t\in\F_{r,N}$ and decoders $\dec_f$ achieve error at most $\eps$ for every function, and let $t_j$ be the number of selected thresholds of coordinate $j$.
These thresholds cut $[N]$ into at most $t_j+1$ intervals of consecutive integers, and the partition generated by all selected functions is the product of these interval partitions over the coordinates.

Fix a coordinate $j$ and an interval $I$ of $\ell\ge2$ points, and choose $a$ such that $f_{j,a}$ labels $\lfloor\ell/2\rfloor$ points of $I$ with one and $\lceil\ell/2\rceil$ with zero.
On every cell whose $j$-th factor is $I$, the coordinate $X_j$ is uniform on $I$ and independent of the other coordinates, so any decoder errs on $f_{j,a}$ with conditional probability at least $\lfloor\ell/2\rfloor/\ell\ge(\ell-1)/(2\ell)$.
These cells have total mass $\ell/N$, so the error on $f_{j,a}$ is at least $(\ell-1)/(2N)$.
Error at most $\eps$ therefore forces $\ell\le2\eps N+1\le3/2$, that is, every interval is a single point.
Hence $t_j\ge N-1$ for every $j$, and $t\ge r(N-1)\ge r/(8\eps)$ when $\eps<c_0$ for a small enough constant $c_0$.
\end{proof}

So the worst-case number of selected functions lies between $\Omega(d^{*}/\eps)$ and $O((d^{*}/\eps)\log(1/\eps))$.

\section{Learning from samples: algorithms and proofs}
\label{app:samples}

\subsection{The empirical entropy-greedy learner}

\suppressfloats[t]
\begin{algorithm2e}[htbp]
  \caption{\textsc{SampleGreedy}$(\F,S,\eps)$}
  \label{alg:sample-greedy}
  \LinesNumbered
  \DontPrintSemicolon
  $\widehat D_S\gets$ uniform distribution on the sample multiset\;
  $(\mathbf g,\Enc,\cdot)\gets$ \textsc{EntropyGreedy}$(\F,\widehat D_S,\eps/4)$\;
  \For{each $f\in\F$}{
    $\dec_{S,f}(z)\gets$ majority of $f$ among samples with code $z$; use $0$ on ties or empty cells\;
  }
  \Return{$(\Enc,(\dec_{S,f})_{f\in\F})$}\;
\end{algorithm2e}

For an infinite class, the decoder family is specified by its majority rule rather than enumerated.
The same oracle qualification as for \Cref{alg:greedy} applies to selecting functions.
By \Cref{thm:det-upper}, the selected list has fewer than
\begin{equation}
  \label{eq:tmax}
  t_{\max}=\left\lceil\frac{16d^{*}}{\eps}\log\frac4\eps\right\rceil
\end{equation}
functions, and every function has empirical error at most $\eps/4$.
To control the population error, define
\[
  \Lcal_t=\left\{x\mapsto\1\{\beta(g_1(x),\ldots,g_{t'}(x))\ne f(x)\}:
  \begin{array}{l}
    0\le t'\le t,\quad f,g_1,\ldots,g_{t'}\in\F,\\
    \beta:\{0,1\}^{t'}\to\{0,1\}
  \end{array}\right\}.
\]

\begin{lemma}[Complexity of the error events]
  \label{lem:loss-class-vc}
  There is an absolute constant $C$ such that, for every integer $t\ge0$,
  \[
    \VCdim(\Lcal_t)\le V(t):=
    C\left[d(t+1)\log\bigl(2+d(t+1)+\Gdual(t+1)\bigr)+\Gdual(t+1)\right].
  \]
\end{lemma}

\begin{proof}
  \label{app:loss-class}
  Fix $m$ points and an integer $t'\le t$.
  The restrictions of $f,g_1,\ldots,g_{t'}$ can be chosen in at most $\Gprim(m)^{t'+1}$ ways.
  Once these functions are fixed, the points realize at most $\Gdual(t')$ joint patterns of $g_1,\ldots,g_{t'}$.
  The decoder $\beta$ therefore induces at most $2^{\Gdual(t')}$ labelings of these points.
  Summing over $t'$, the growth function of $\Lcal_t$ is at most
  \[
    (t+1)\Gprim(m)^{t+1}2^{\Gdual(t+1)}.
  \]
  If the points are shattered and $m\ge d$, Sauer--Shelah gives
  \[
    m\le\log(t+1)+d(t+1)\log\frac{em}{d}+\Gdual(t+1).
  \]
  Put $a=d(t+1)$ and $B=2+a+\Gdual(t+1)$.
  For a sufficiently large absolute constant $C$, at $m=C(a\log B+\Gdual(t+1))$ the right side is less than $m$: indeed $m\le CB\log B$, so $\log(em/d)\le O(\log B+\log C)$.
  Moreover, the derivative of the right side with respect to $m$ is $a/(m\ln2)<1$ beyond this point.
  The inequality therefore fails for all larger $m$.
  The case $m<d$ is already covered by the claimed bound.
\end{proof}

We use the following standard relative form of VC uniform convergence \citep{vapnik1971uniform,anthony1993result}.

\begin{fact}[Relative deviation bound]
  \label{fact:relative}
  Let $\mathcal A$ be a class of events of VC dimension at most $V$ satisfying the measurability assumptions of \Cref{sec:samples}.
  For an absolute constant $C$, if $0<\eta\le1/4$, $0<\delta<1$, and
  \[
    n\ge\frac{C}{\eta}\left(V\log\frac1\eta+\log\frac1\delta\right),
  \]
  then, with probability at least $1-\delta$, every $A\in\mathcal A$ satisfies
  $D(A)\le2\widehat D_S(A)+\eta/2$.
\end{fact}

\begin{proof}[Proof of \Cref{thm:sample-optimal}]
  Each error event of \Cref{alg:sample-greedy} belongs to $\Lcal_{t_{\max}}$ and has empirical probability at most $\eps/4$.
  By \Cref{lem:loss-class-vc,fact:relative}, the sharper, class-dependent condition
  \begin{equation}
    \label{eq:sample-growth-bound}
    n\ge\frac{C}{\eps}\left(V(t_{\max})\log\frac1\eps+\log\frac1\delta\right)
  \end{equation}
  makes every population error at most $2(\eps/4)+\eps/2=\eps$.

  To obtain the displayed dimension-only bound, write $r=d^{*}$ and $L=\log(1/\eps)\ge2$.
  We have $t_{\max}+1\le C_1rL/\eps$ and
  $\Gdual(t_{\max}+1)\le(C_2L/\eps)^r$ by Sauer--Shelah.
  Also $d<2^{r+1}$ by Assouad's inequality, so
  \[
    \log\bigl(2+d(t_{\max}+1)+\Gdual(t_{\max}+1)\bigr)\le C_3rL.
  \]
  Consequently,
  \[
    \frac{L}{\eps}V(t_{\max})
    \le C_4\left[\frac{dr^2L^3}{\eps^2}
      +\left(\frac{C_2L}{\eps}\right)^{r+1}\right]
    \le\left(\frac{C_5}{\eps}\right)^{r+1}L^{r+2}.
  \]
  The last inequality uses $dr^2\le C_6^{r+1}$ and $r\ge1$.
  Enlarging the constant in the theorem therefore implies \eqref{eq:sample-growth-bound}.

  Finally, the selected list has length $t<t_{\max}$, hence $t<(16r/\eps)\log(4/\eps)$.
  If $t\ge r$, Sauer--Shelah bounds its number of patterns by $(et/r)^r$; if $t<r$, it has at most $2^t$ patterns.
  In both cases the bit length is at most
  \[
    \left\lceil r\log\left(\frac{16e}{\eps}\log\frac4\eps\right)\right\rceil
    =O(r\log(1/\eps)).
  \]
\end{proof}

\subsection{The empirical cover learner}

\suppressfloats[t]
\begin{algorithm2e}[htbp]
  \caption{\textsc{CoverRepresentation}$(\F,S,\eps)$}
  \label{alg:cover}
  \LinesNumbered
  \DontPrintSemicolon
  $\widehat\rho_S(f,g)\gets n^{-1}\sum_{i=1}^{n}\1\{f(X_i)\ne g(X_i)\}$\;
  $\F_0\gets$ a maximal subset of $\F$ with pairwise $\widehat\rho_S$-distance at least $\eps/8$\;
  $\Enc(x)\gets$ index of the joint pattern $(g(x))_{g\in\F_0}$\;
  \For{each $f\in\F$}{
    Choose $g_f\in\F_0$ with $\widehat\rho_S(f,g_f)<\eps/8$\;
    $\dec_{S,f}(\Enc(x))\gets g_f(x)$\;
  }
  \Return{$(\Enc,(\dec_{S,f})_{f\in\F})$}\;
\end{algorithm2e}

\begin{proof}[Proof of \Cref{thm:cover-protocol}]
  Haussler's packing bound \citep{haussler1995sphere} for the empirical distribution gives $m:=|\F_0|\le(C'/\eps)^d$.
  In particular, a finite maximal set exists, and maximality gives the required $g_f$ for every $f$.
  Disagreements between pairs from $\F$ have growth function at most $\Gprim(s)^2$ on $s$ points and hence VC dimension $O(d)$.
  By \Cref{fact:relative}, simultaneously for every pair,
  \[
    \Pr_D[f(X)\ne g_f(X)]\le2\widehat\rho_S(f,g_f)+\eps/2<\eps.
  \]
  The representation has at most $\Gdual(m)$ values.
  If $m\ge d^{*}$, its bit length is at most $\lceil d^{*}\log(em/d^{*})\rceil$ by Sauer--Shelah.
  If $m<d^{*}$, its bit length is at most $m$.
  Both cases give $O(dd^{*}\log(1/\eps))$ bits.
\end{proof}

\begin{corollary}[Matching sample rate for thresholds]
  \label{cor:threshold-samples}
  For thresholds on a line, \Cref{alg:cover} uses at most $2+8/\eps$ representation values and has error $\eps$ with confidence $1-\delta$ from
  $O((\log(1/\eps)+\log(1/\delta))/\eps)$ samples.
\end{corollary}

\begin{proof}
  Order the members of $\F_0$ by their thresholds.
  The empirical disagreement of two consecutive members is the difference of their empirical positive masses.
  These disjoint increments sum to at most one, and each is at least $\eps/8$, so $|\F_0|-1\le8/\eps$.
  Nested thresholds realize at most $|\F_0|+1$ patterns, corresponding to intervals between consecutive cut points.
  The sample bound follows from \Cref{thm:cover-protocol} with $d=1$.
\end{proof}

\subsection{Proof of the threshold sample lower bound}
\label{app:sample-lower}

\begin{proof}[Proof of \Cref{thm:sample-lower}]
  Set $K=\lfloor1/(2\eps)\rfloor$ and $N=4M+K$.
  Choose a uniformly random $K$-subset $A\subseteq[N]$ and let $D_A$ be uniform on $A$.
  Since $1/K>\eps$, success requires every binary decoder to be correct at every point of $A$.

  Condition on the sample $S$ and on the learner's random seed.
  The output representation and decoders are now fixed.
  If there are more than $M$ representation values, the learner has already failed the length requirement.
  Otherwise, let $B\subseteq[N]$ consist of the points at which every threshold is decoded correctly.
  Distinct points have distinct threshold label vectors, so each representation value accounts for at most one point of $B$ and $|B|\le M$.

  Write $\mathcal O$ for the observed support points and $U=A\setminus\mathcal O$ for the unseen ones.
  Given $S$, the set $U$ is uniform among all $(K-|\mathcal O|)$-subsets of $[N]\setminus\mathcal O$.
  This remains true after conditioning on the independent random seed.
  If $U$ is nonempty, the conditional probability that $U\subseteq B$, as success requires, is at most
  \[
    \frac{|B\setminus\mathcal O|}{N-|\mathcal O|}
    \le\frac{M}{N-K}\le\frac14.
  \]
  Indeed, the first bound follows by selecting a uniformly random member of $U$.

  It remains to bound the probability of seeing every support point.
  Let $\upsilon=|U|$.
  If $n\le\frac14K\ln K$, then
  $\E\upsilon=K(1-1/K)^n\ge c_0K^{3/4}$ for all sufficiently large $K$.
  The unseen-point indicators are negatively correlated, because
  $(1-2/K)^n\le(1-1/K)^{2n}$, so $\Var(\upsilon)\le\E\upsilon$.
  Chebyshev's inequality yields $\Pr[\upsilon=0]\le1/\E\upsilon=o(1)$.
  Thus the success probability, averaged over $A$, the sample, and the learner's randomness, is at most $1/4+o(1)<2/3$.
  Some fixed $A$ therefore violates the required confidence.
  Since $K=\Theta(1/\eps)$, this proves the claimed bound with an absolute constant $c$.
\end{proof}

\section{Extensions}
\label{sec:extensions}

This appendix considers four variations of the model.
Thresholding extends the upper bound to real-valued functions, shared randomness gives the familiar dual-cover characterization, and a message from the function holder gives an upper bound in terms of primal VC dimension.
For halfspaces, a classical geometric partition theorem gives a concrete alternative construction.

\subsection{Real-valued functions}
\label{sec:real}

Let $\F\subseteq[0,1]^{\X}$ be a class of measurable functions.
In this subsection, we extend \Cref{def:protocol} by allowing a decoder $\Dec:\{0,1\}^{b}\times\F\to[0,1]$ and requiring
\[
  \sup_{f\in\F}\E_{X\sim D}\bigl|f(X)-\Dec(\Enc(X),f)\bigr|\le\eps.
\]
We use $\CC^{\to}_{\det}(\F,D,\eps)$ for the minimum length in this extended model.
For binary tasks, allowing these outputs does not change the minimum length: on a cell where $f(X)$ has mean $p$, the expected absolute error of a prediction $a\in[0,1]$ is $p+a(1-2p)$, which is minimized by a binary majority label.

The dual class $\F^{*}=\{\ev_x:x\in\X\}$ is now real-valued, and we write $d^{*}_{\mathrm P}=\Pdim(\F^{*})$ for its pseudo-dimension.
Thus $d^{*}_{\mathrm P}$ is the largest number of functions $f_1,\ldots,f_k$ for which there are thresholds $\theta_1,\ldots,\theta_k$ such that the vectors $(\1\{f_i(x)\ge\theta_i\})_{i=1}^{k}$ realize every binary pattern as $x$ varies.
For $f\in\F$ and $\theta\in[0,1]$, let $f^{\ge\theta}(x)=\1\{f(x)\ge\theta\}$, and let $\F^{\ge}=\{f^{\ge\theta}:f\in\F,\ \theta\in[0,1]\}$.

\begin{lemma}[Threshold lifting]
  \label{lem:lifting}
  $\VCdim\bigl((\F^{\ge})^{*}\bigr)=d^{*}_{\mathrm P}$.
\end{lemma}

\begin{proof}
  A set of functions $f_1^{\ge\theta_1},\ldots,f_k^{\ge\theta_k}$ is dual-shattered exactly when for every $\sigma\in\{0,1\}^{k}$ some $x$ has $\1\{\ev_x(f_i)\ge\theta_i\}=\sigma_i$ for all $i$.
  Two thresholds of the same function are nested, so they cannot be dual-shattered together.
  Hence the underlying $f_i$ are distinct.
  The condition is then the definition of pseudo-shattering $f_1,\ldots,f_k$ by $\F^{*}$ with witnesses $\theta_1,\ldots,\theta_k$.
\end{proof}

\begin{theorem}[Real-valued upper bound]
  \label{thm:real-upper}
  If $d^{*}_{\mathrm P}<\infty$, then for every distribution $D$ and $0<\eps\le1/4$, $\CC^{\to}_{\det}(\F,D,\eps)\le4d^{*}_{\mathrm P}\log(1/\eps)$.
\end{theorem}

\begin{proof}
  If $d^{*}_{\mathrm P}=0$, every $f$ is constant and zero bits suffice.
  Assume $d^{*}_{\mathrm P}\ge1$.
  By \Cref{lem:lifting,thm:det-upper} applied to $\F^{\ge}$, there is a partition $\Pcal$ with at most $(4e\log(1/\eps)/\eps)^{d^{*}_{\mathrm P}}$ cells on which every $f^{\ge\theta}$ has majority error at most $\eps$.
  On a cell $A$ of positive mass, let the decoder output a median $m_A(f)$ of $f(X)$ given $X\in A$; its value on zero-mass cells is arbitrary.
  A median gives compatible majority decisions at all thresholds: if $\theta\le m_A(f)$ then $\Pr[f(X)\ge\theta\mid A]\ge1/2$, and if $\theta>m_A(f)$ then $\Pr[f(X)\ge\theta\mid A]\le\Pr[f(X)>m_A(f)\mid A]\le1/2$.
  By the identity $|y-y'|=\int_0^1|\1\{y\ge\theta\}-\1\{y'\ge\theta\}|\,d\theta$ for $y,y'\in[0,1]$,
  \begin{align*}
    \E\bigl|f(X)-m_{\Pcal(X)}(f)\bigr|
    &=\int_0^1\E\bigl|f^{\ge\theta}(X)-\1\{m_{\Pcal(X)}(f)\ge\theta\}\bigr|\,d\theta\\
    &\le\eps,
  \end{align*}
  where $\Pcal(X)$ is the cell of $X$.
  The length bound follows as in \Cref{thm:det-upper}.
\end{proof}

Binary classes are special cases, and the pseudo-dimension of a binary class is its VC dimension, so \Cref{thm:threshold-lower} shows that \Cref{thm:real-upper} is tight in the worst case.

\begin{remark}[Scale-sensitive bounds]
  Pseudo-dimension does not distinguish large variations in function values from small ones.
  A bound in terms of the dual fat-shattering dimensions at the relevant scales could therefore be sharper; we do not pursue such a bound here.
\end{remark}

\subsection{Randomized representations and dual metric entropy}
\label{sec:public}

We return to binary classes and binary decoder outputs.
A \emph{public-coin} representation draws a seed $\omega$ that both parties see, independently of $X$, and then uses deterministic maps $(\Enc_\omega,\Dec_\omega)$ of length $b$.
Its error is $\sup_{f\in\F}\E_\omega\E_{X\sim D}|f(X)-\Dec_\omega(\Enc_\omega(X),f)|$, and $\CC^{\to}_{\pub}(\F,D,\eps)$ is the least $b$ that achieves error $\eps$.
We work with finite $\X$ and $\F$ so that the minimax theorem applies without topological conditions, and we let $\mathfrak R_\eps(\F)=\sup_D\CC^{\to}_{\pub}(\F,D,\eps)$.
By the minimax principle, a public-coin representation that serves every $f$ exists exactly when, for every distribution $Q$ on functions, some deterministic representation achieves the same error averaged over $f\sim Q$.
The latter is the product-distribution setting of \citet{kremer1999randomized} and \citet{molinaro2015amplification}, since $f\sim Q$ and $X\sim D$ are independent.
The results below restate this connection and the packing characterization in our notation; they are included for comparison with the deterministic theorem.

For a distribution $Q$ on $\F$, the \emph{dual distance} is $\rho_Q(x,x')=\E_{f\sim Q}|f(x)-f(x')|$.
It is the probability that a random function drawn from $Q$ distinguishes $x$ from $x'$, and is a pseudometric because distinct inputs may have distance zero.
Let $\mathrm{Cov}(\eps,Q)$ be the least size of a set $Z\subseteq\X$ such that every $x\in\X$ has $\rho_Q(x,z)\le\eps$ for some $z\in Z$, and say that $x_1,\ldots,x_M$ form a \emph{strict $\gamma$-packing} if $\rho_Q(x_i,x_j)>\gamma$ for all $i\neq j$.

\begin{theorem}[Upper bound by dual covers]
  \label{thm:public-cover}
  For every finite class $\F$ on a finite domain,
  \[
    \mathfrak R_\eps(\F)\le\max_Q\lceil\log\mathrm{Cov}(\eps,Q)\rceil.
  \]
\end{theorem}

\begin{proof}
  Fix $D$ and $Q$, and let $Z$ be a cover of size $\mathrm{Cov}(\eps,Q)$.
  The encoder sends the index of a center $z(x)\in Z$ with $\rho_Q(x,z(x))\le\eps$, and the decoder outputs $f(z(x))$.
  Then $\E_{f\sim Q}\E_{X\sim D}|f(X)-f(z(X))|=\E_{X\sim D}\rho_Q(X,z(X))\le\eps$.
  Set $b=\max_Q\lceil\log\mathrm{Cov}(\eps,Q)\rceil$ and pad shorter encodings to length $b$.
  Consider the finite zero-sum game in which one player picks $f\in\F$, the other picks a length-$b$ deterministic representation with a binary decoder, and the payoff is the error on $f$.
  We just showed that every mixed strategy $Q$ of the first player has a pure response of payoff at most $\eps$.
  By the minimax theorem \citep[see][]{yao1977probabilistic}, some mixture of representations has payoff at most $\eps$ against every $f$, and this mixture is a public-coin universal representation.
\end{proof}

Haussler's packing bound \citep{haussler1995sphere} applied to $\F^{*}$ gives $\mathrm{Cov}(\eps,Q)\le(C/\eps)^{d^{*}}$ for every $Q$, hence $\mathfrak R_\eps(\F)\le d^{*}\log(C/\eps)+1$.

\begin{theorem}[Packing lower bound]
  \label{thm:public-packing}
  Let $Q$ be a distribution on $\F$, let $x_1,\ldots,x_M$ form a strict $4\eps$-packing for $\rho_Q$, and let $D$ be uniform on them.
  Every public-coin universal representation with error at most $\eps$ has length at least $\log M-1$.
\end{theorem}

\begin{proof}
  Averaging the guarantee over $f\sim Q$ and the seed gives a seed $\omega$ with $\E_{f\sim Q}\E_{X\sim D}|f(X)-\Dec_\omega(\Enc_\omega(X),f)|\le\eps$.
  For this seed, let $\eta_i=\E_{f\sim Q}|f(x_i)-\Dec_\omega(\Enc_\omega(x_i),f)|$; by Markov's inequality, at least $M/2$ indices have $\eta_i\le2\eps$.
  If two such points lie in the same cell, the triangle inequality gives $\rho_Q(x_i,x_j)\le\eta_i+\eta_j\le4\eps$, a contradiction.
  So there are at least $M/2$ cells.
\end{proof}

A maximal strict $4\eps$-packing is a cover at radius $4\eps$.
Thus the two bounds characterize public-coin complexity up to a constant change in accuracy and one bit of rounding.

\begin{corollary}[Metric-entropy characterization, adapted from \citealp{molinaro2015amplification}]
  \label{cor:metric-characterization}
  For every finite class $\F$ on a finite domain,
  \[
    \max_Q\log\mathrm{Cov}(4\eps,Q)-1\le\mathfrak R_\eps(\F)\le\max_Q\lceil\log\mathrm{Cov}(\eps,Q)\rceil .
  \]
\end{corollary}

\begin{corollary}[Worst-case public-coin complexity]
  \label{cor:public-tight}
  There is $\eps_0>0$ such that for every $d^{*}\ge1$ and $0<\eps\le\eps_0$, $\sup_{\F,D}\CC^{\to}_{\pub}(\F,D,\eps)=\Theta(d^{*}\log(1/\eps))$, where the supremum is over classes with $\VCdim(\F^{*})\le d^{*}$.
\end{corollary}

\begin{proof}
  The upper bound follows from \Cref{thm:det-upper}, since a deterministic representation is a public-coin representation.
  For the lower bound, the proof of \Cref{thm:threshold-lower} uses only the error averaged uniformly over the functions of $\F_{r,N}$, and some seed of a public-coin representation achieves this average error.
\end{proof}

The following multiplicative-weights argument limits the number of representations needed in the mixture.

\begin{proposition}[Small-support mixtures]
  \label{prop:mw}
  Let $0<\eps<1$, let $\F=\{f^{1},\ldots,f^{m}\}$ with $m\ge2$, and suppose that for every distribution $Q$ on $\F$ some deterministic representation $\kappa$ of length $b$ has $\sum_iQ(f^{i})\,\ell(i,\kappa)\le\eps$, where $\ell(i,\kappa)\in[0,1]$ is its error on $f^{i}$.
  Then there are $T=\lceil4\ln m/\eps\rceil$ deterministic representations of length $b$ whose uniform mixture has error at most $3\eps$ on every $f^{i}$.
\end{proposition}

\begin{proof}
  Start from weights $w_1(i)=1$ and let $Q_\tau(f^{i})=w_\tau(i)/\sum_jw_\tau(j)$.
  In round $\tau$, choose $\kappa_\tau$ with $\sum_iQ_\tau(f^{i})\ell(i,\kappa_\tau)\le\eps$ and set $w_{\tau+1}(i)=w_\tau(i)(1+\frac12\ell(i,\kappa_\tau))$.
  The total weight $\Omega_\tau=\sum_iw_\tau(i)$ satisfies $\Omega_{\tau+1}\le\Omega_\tau(1+\eps/2)$, so $\ln\Omega_{T+1}\le\ln m+\eps T/2$.
  Since $\ln(1+y)\ge y-y^{2}$ for $y\ge0$ and $\ell\le1$, $\ln w_{T+1}(i)\ge\frac14\sum_\tau\ell(i,\kappa_\tau)$.
  Hence $\frac1T\sum_\tau\ell(i,\kappa_\tau)\le2\eps+4\ln m/T\le3\eps$.
\end{proof}

\subsection{Interaction}
\label{sec:interaction}

The one-way model forces the encoder to summarize the label vector of $x$ across all functions.
If the function holder may speak first, a cover of the primal class gives a different upper bound.

\begin{theorem}[Two-message upper bound]
  \label{thm:interactive}
  For every distribution $D$ and $0<\eps<1/2$, there is a deterministic two-message protocol, in which the function holder speaks first, with error at most $\eps$ and $d\log(C/\eps)+1$ bits of communication, where $C$ is an absolute constant.
\end{theorem}

\begin{proof}
  By Haussler's bound, $\F$ has a subset $\F_\eps$ of size at most $(C/\eps)^{d}$ such that every $f$ has some $g\in\F_\eps$ with $\Pr_D[f(X)\neq g(X)]\le\eps$.
  The function holder sends the index of such a $g$, and the input holder replies with the bit $g(x)$.
  This costs $\lceil\log|\F_\eps|\rceil+1$ bits; increasing the absolute constant $C$ absorbs the ceiling.
\end{proof}

\begin{example}[Coordinates]
  \label{ex:coordinates}
  For the coordinate functions on $\{0,1\}^{r}$, $d^{*}=r$, so under the uniform distribution a one-way representation needs $\Omega(r)$ bits even at constant error (\Cref{prop:instance-lower}).
  With interaction, the function holder sends $i$ and the input holder replies with $x_i$, which uses $\lceil\log r\rceil+1$ bits and has zero error.
\end{example}

\subsection{Halfspaces}
\label{sec:halfspaces}

For halfspaces under a uniform distribution on finitely many points, the partition theorem of \citet{matousek1992efficient} gives groups of points contained in simplices.
These groups, rather than the simplices themselves, are the cells of the representation.
A hyperplane \emph{crosses} a relatively open simplex if it intersects the simplex without containing it.

\begin{theorem}[Simplicial partitions, \citealp{matousek1992efficient}]
  \label{thm:matousek}
  For every $r\ge2$ there is $c_r\ge1$ such that the following holds.
  Let $P\subset\R^{r}$ be finite, let $s$ be an integer with $2\le s<|P|$, and let $\nu=|P|/s$.
  Then $P$ can be partitioned into classes $P_i$ with $s\le|P_i|<2s$, each contained in a relatively open simplex $\sigma_i$, such that every hyperplane crosses at most $c_r\nu^{1-1/r}$ of the simplices.
\end{theorem}

\begin{proposition}[Explicit representation for halfspaces]
  \label{prop:halfspace}
  Let $\F$ be the class of closed affine halfspaces in $\R^{r}$ with $r\ge2$, let $0<\eps<1$, and let $D$ be uniform on a finite set $P$ with $|P|\ge2(c_r/\eps)^{r}$.
  Then $\CC^{\to}_{\det}(\F,D,\eps)\le r\log(c_r/\eps)+2$.
  On the support $P$, the encoder cells are the groups $P_i\subseteq\sigma_i$ supplied by \Cref{thm:matousek}.
\end{proposition}

\begin{proof}
  Let $n=|P|$ and apply \Cref{thm:matousek} with $s=\lfloor n(\eps/c_r)^{r}\rfloor$, which satisfies $2\le s<n$; then $\nu=n/s\ge(c_r/\eps)^{r}$.
  If the boundary hyperplane of a halfspace does not cross $\sigma_i$, then $\sigma_i$ lies on one side of it or inside it, and the halfspace is constant on $P_i$.
  A crossed class has mass less than $2s/n=2/\nu$, so its minority mass is less than $1/\nu$.
  Hence the majority error is at most $c_r\nu^{1-1/r}/\nu=c_r\nu^{-1/r}\le\eps$.
  The number of classes is at most $n/s\le2(c_r/\eps)^{r}$, because $s\ge n(\eps/c_r)^{r}/2$ when $n(\eps/c_r)^{r}\ge2$, and \Cref{prop:partition} gives at most $r\log(c_r/\eps)+2$ bits.
  Assign inputs outside $P$ an arbitrary codeword; they have zero probability under $D$.
\end{proof}

The dual VC dimension of affine halfspaces in $\R^{r}$ is at most $r+1$, so \Cref{thm:det-upper} already gives $O(r\log(1/\eps))$ bits with an absolute constant.
The geometric construction instead explains the error through hyperplane crossings.
The containing simplices need not be disjoint, and no convexity of the encoder cells is claimed.

\section{Universal representations as learning objectives}
\label{sec:objectives}
\label{app:objectives}

The entropy potential has an exact interpretation as optimal reconstruction log loss.
We show that maximizing this entropy over sets of selected functions also yields a universal representation, and explain why both the choice of representation class and the reconstruction objective matter.
Finally, we relate universal representations to the multicalibration of conditional-mean predictors under arbitrary labelings.

\subsection{Self-reconstruction}
\label{sec:reconstruction}

Assume that $\X$ is finite.
A reconstruction predictor, separate from the binary task decoder, predicts a distribution over inputs from the representation $\Pcal(X)$.
Equivalently, it is a map $\phi:\X\to\Delta(\X)$ that is constant on every cell, with loss $\mathrm{LL}_D(\phi)=-\E_{X\sim D}\log\phi(X)_X$.
Here $\phi(x)_{x'}$ is the probability assigned to reconstructing $x'$ after observing the cell of $x$.

\begin{proposition}[Optimal self-reconstruction loss]
  \label{prop:self-reconstruction}
  The best predictor that is constant on each cell is $\phi(x)=D(\cdot\mid\Pcal(X)=\Pcal(x))$, and its loss is $H(X\mid\Pcal(X))=H(X)-H(\Pcal(X))$.
\end{proposition}

\begin{proof}
  By the standard entropy characterization of log loss \citep{cover2006elements}, on each cell, log loss is minimized by the conditional distribution, whose loss is the conditional entropy; averaging over cells gives $H(X\mid\Pcal(X))$.
  Since $\Pcal(X)$ is a function of $X$, the chain rule gives $H(X\mid\Pcal(X))=H(X)-H(\Pcal(X))$.
\end{proof}

So each step of \Cref{alg:greedy}, which raises $H(Z_{\mathbf g}(X))$ by more than $2\eps$, lowers the optimal self-reconstruction loss by more than $2\eps$ bits.

This identity alone does not make arbitrary reconstruction-optimal encoders universal.
For example, let $X=(A,B)$ consist of independent fair bits and let $\F=\{(a,b)\mapsto a\}$.
Among all one-bit encoders, $Z=B$ maximizes $H(Z)=1$ and minimizes the optimal reconstruction log loss $H(X\mid Z)=1$, yet every decoder of $A$ from $Z$ has error $1/2$.
The next result therefore optimizes over representations generated by members of $\F$, with a budget on how many functions may be selected.

\subsection{One-shot entropy maximization}
\label{sec:one-shot}
\label{app:one-shot}

With a large enough selection budget, every joint-entropy maximizer gives the required guarantee.

\begin{theorem}[One-shot entropy maximization]
  \label{thm:one-shot}
  Let $\X$ be finite and $0<\eps<1/2$.
  There is an absolute constant $C$ such that, for every integer $t\ge C(d^{*}/\eps)\log(1/\eps)$, every maximizer of $H((g(X))_{g\in\mathcal G})$ over sets $\mathcal G\subseteq\F$ with $|\mathcal G|\le t$ satisfies $\minerr_{\mathcal G}(f)\le\eps$ for every $f\in\F$.
\end{theorem}

\begin{proof}
  Write $Z_{\mathcal G}=(g(X))_{g\in\mathcal G}$ and $H^{\max}(t)=\max_{|\mathcal G|\le t}H(Z_{\mathcal G})$.
  The maximum exists because $\X$ is finite and hence $\F$ is finite.
  Let $\mathcal G$ attain it, and suppose that some $f\in\F$ has $\minerr_{\mathcal G}(f)>\eps$.
  By \Cref{lem:entropy-increment}, adding $f$ raises the entropy by more than $2\eps$.
  If $|\mathcal G|<t$, this contradicts maximality immediately.
  Otherwise $|\mathcal G|=t$, $f\notin\mathcal G$, and
  \[
    H^{\max}(t+1)>H^{\max}(t)+\eps.
  \]
  In particular, $|\F|\ge t+1$.
  Choose an entropy-maximizing set $\mathcal T$ of exactly $t+1$ functions; padding any smaller maximizer is possible and cannot decrease entropy.
  If $\mathcal R$ is a uniformly random $t$-subset of $\mathcal T$, Han's inequality \citep{han1978nonnegative} gives
  \[
    H^{\max}(t)\ge\E H(Z_{\mathcal R})\ge\frac{t}{t+1}H^{\max}(t+1).
  \]
  Consequently, the dual Sauer bound yields
  \[
    H^{\max}(t+1)-H^{\max}(t)
    \le\frac{H^{\max}(t+1)}{t+1}
    \le\frac{d^{*}}{t+1}\log\frac{e(t+1)}{d^{*}}.
  \]
  For a sufficiently large absolute $C$, the last expression is at most $\eps$ when $t\ge C(d^{*}/\eps)\log(1/\eps)$, by the calculation in \Cref{lem:solve-stopping}.
  This is a contradiction.
\end{proof}

\subsection{A coordinate-wise objective fails}
\label{sec:mi}

When the input is a vector $X=(X_1,\ldots,X_n)$, it is tempting to replace the joint entropy of the selected functions by the total information they carry about the individual coordinates,
\[
  \Phi(\mathcal G)=\sum_{j=1}^{n}I\bigl(X_j;(g(X))_{g\in\mathcal G}\bigr),
\]
Here $I(U;V)=H(U)-H(U\mid V)$ is mutual information.
Maximizing $\Phi$ is equivalent to minimizing the sum of the optimal log losses for reconstructing the coordinates separately.
The \emph{$\Phi$-greedy procedure} starts from $\mathcal G=\emptyset$ and, while some function has majority error more than $\eps$, adds an unselected function $g\in\F$ that maximizes $\Phi(\mathcal G\cup\{g\})$.
We allow this rule to select a function whose current majority error is already at most $\eps$.

\begin{theorem}[Coordinate-wise information does not suffice]
  \label{thm:mi-counterexample}
  Let $\eps=0.01$.
  \begin{enumerate}[label=(\alph*),leftmargin=2em]
    \item For every integer $N\ge3$ there are a distribution on $\{0,1\}^{N+2}$ and a class $\F$ with $d^{*}=2$ on which the $\Phi$-greedy procedure selects at least $N$ functions, and for every $t\le N-1$ no $t$-set maximizing $\Phi$ is a universal representation with error $\eps$.
    \item For every integer $\lambda\ge1$ there are $n$, a distribution of $X$ on $\{0,1\}^{n}$ with $H(X)\le2$, and $\F=\{x\mapsto x_j:j\in[n]\}$ on which the $\Phi$-greedy procedure selects at least $\lambda+1$ functions, and for every $t\le\lambda$ no $t$-set maximizing $\Phi$ is a universal representation with error $\eps$.
  \end{enumerate}
\end{theorem}

In (a), nested thresholds each carry information about some coordinate, while an indispensable parity bit carries no information about any coordinate separately.
In (b), rare duplicated bits outweigh a single fair bit.
\Cref{app:mi} gives the constructions.
These counterexamples show that coordinate information alone need not prioritize functions whose current majority error exceeds $\eps$.
If we restrict each choice to such a function, \Cref{lem:entropy-increment} guarantees an increase of more than $2\eps$ in joint entropy at every step, regardless of the choice rule.
The argument for \Cref{thm:det-upper} then bounds the number of selections by $O((d^{*}/\eps)\log(1/\eps))$, including when we use $\Phi$ to choose among these functions.

\subsection{Multicalibration under relabeling}
\label{sec:multicalibration}
\label{app:umc}

Fix a distribution $D$ on $\X$ and a representation $Z=Z(X)$ with finitely many values.
For a joint distribution $P$ of $(X,Y)$ with $Y\in\{0,1\}$ and $X$-marginal $D$, let $\varphi_P=\E_P[Y\mid Z]$ be the conditional-mean predictor.
Following the averaged formulation of multicalibration \citep{gopalan2022omnipredictors}, for a binary function $f$ define its mass-weighted $L_1$ multicalibration error by
\[
  \MC_f(\varphi_P;P)=\E_P\Bigl|\E_P\bigl[f(X)(Y-\varphi_P)\bigm|\varphi_P\bigr]\Bigr| .
\]
This averages the magnitude of the residual bias for the group $f=1$ across prediction values, weighting each group and prediction value by its probability.
It does not require a uniform error bound in every prediction bin.
We ask whether the conditional-mean predictor has small error under every labeling $P$.
Write $\mu_f(Z)=\E[f(X)\mid Z]$.

\begin{theorem}[Relabeling-robust multicalibration]
  \label{thm:umc}
  For every representation $Z$ and binary function $f$,
  \[
    \sup_{P}\MC_f(\varphi_P;P)=\E\bigl[\mu_f(Z)(1-\mu_f(Z))\bigr]=\E\bigl[\Var(f(X)\mid Z)\bigr],
  \]
  where the supremum is over joint distributions $P$ of $(X,Y)$ with $X$-marginal $D$ and $Y\in\{0,1\}$.
\end{theorem}

\begin{proof}
  Fix $f$ and $P$.
  Condition on a value of $Z$, write $\mu=\Pr[f(X)=1\mid Z]$, and let $\alpha=\E_P[Y\mid f(X)=1,Z]$ and $\beta=\E_P[Y\mid f(X)=0,Z]$, with arbitrary values in $[0,1]$ when the conditioning event has probability zero.
  Then $\varphi_P=\mu\alpha+(1-\mu)\beta$ and
  \[
    \E_P\bigl[f(X)(Y-\varphi_P)\mid Z\bigr]=\mu\alpha-\mu\varphi_P=\mu(1-\mu)(\alpha-\beta),
  \]
  whose absolute value is at most $\mu(1-\mu)$.
  Write $W=\E_P[f(X)(Y-\varphi_P)\mid Z]$.
  Since $\varphi_P$ is a function of $Z$, the tower property and Jensen's inequality give
  \[
    \MC_f(\varphi_P;P)
    =\E_P\bigl|\E_P[W\mid\varphi_P]\bigr|
    \le\E_P|W|
    \le\E\bigl[\mu_f(Z)(1-\mu_f(Z))\bigr].
  \]
  For the reverse inequality, take $Y=f(X)$.
  Then $\varphi_P=\mu_f(Z)$, and conditioning on $\mu_f(Z)$ gives
  \[
    \E[f(X)(f(X)-\mu_f(Z))\mid\mu_f(Z)]
    =\mu_f(Z)(1-\mu_f(Z))\ge0,
  \]
  so equality holds.
  Finally, $\Var(f(X)\mid Z)=\mu_f(Z)(1-\mu_f(Z))$ because $f$ is binary.
\end{proof}

\begin{corollary}[Equivalence with universal representations]
  \label{cor:umc-equivalence}
  For every binary function $f$,
  \[
    \sup_P\MC_f(\varphi_P;P)\le\minerr_Z(f)\le2\sup_P\MC_f(\varphi_P;P).
  \]
  Consequently, let $\mathrm{UMC}_\alpha(\F,D)$ be the least number of bits of a representation $Z$ such that
  \[
    \sup_P\sup_{f\in\F}\MC_f(\varphi_P;P)\le\alpha .
  \]
  Then
  \[
    \CC^{\to}_{\det}(\F,D,2\alpha)\le\mathrm{UMC}_\alpha(\F,D)\le\CC^{\to}_{\det}(\F,D,\alpha).
  \]
\end{corollary}

\begin{proof}
  By \Cref{lem:majority}, $\minerr_Z(f)=\E\min\{\mu_f(Z),1-\mu_f(Z)\}$.
  For $\mu\in[0,1]$, $\mu(1-\mu)\le\min\{\mu,1-\mu\}\le2\mu(1-\mu)$, as one checks for $\mu\le1/2$ using $1-\mu\ge1/2$.
  Averaging and applying \Cref{thm:umc} gives the first claim.
  If $\sup_f\minerr_Z(f)\le\alpha$, then $\sup_P\sup_f\MC_f\le\alpha$, which gives the upper bound on $\mathrm{UMC}_\alpha$; if $\sup_P\sup_f\MC_f\le\alpha$, then $\sup_f\minerr_Z(f)\le2\alpha$, which gives the lower bound.
\end{proof}

Thus the conditional-mean predictor is multicalibrated under every labeling precisely when the representation approximately preserves each member of $\F$, up to a factor of two in the error.
We distinguish this requirement from asking, for each labeling, whether some $Z$-based predictor has small $\MC_f$ error for every $f\in\F$, without also requiring ordinary calibration.
Here we evaluate an arbitrary predictor using the same residual expression that defines $\MC_f$ above.
This weaker requirement can lose the equivalence when $\F$ omits the constant-one test: if $\F=\{f\}$ and $0<D(f=1)<1$, the constant predictor $\E_P[Y\mid f=1]$ has zero $\MC_f$ error for every $P$, even with a constant representation.
That predictor need not be ordinarily calibrated.

If we require ordinary calibration, we retain the lower direction even when we allow an arbitrary predictor $\psi=\psi(Z)\in[0,1]$.
Indeed, for the labeling $Y=f(X)$, calibration means $\E_P[Y\mid\psi]=\psi$, so $\E[f(X)\mid\psi]=\psi$ and
\[
  \MC_f(\psi;P)
  =\E\bigl[\psi(1-\psi)\bigr]
  =\E\Var(f(X)\mid\psi)
  \ge\E\Var(f(X)\mid Z).
\]
The residual conditioned on $\psi$ equals $\psi(1-\psi)\ge0$, and the last inequality follows from conditional total variance because $\psi$ is a function of $Z$.
Thus allowing a calibrated $Z$-based predictor for each labeling still gives the lower bound in \Cref{cor:umc-equivalence}.

\subsection{Coordinate-wise counterexamples}
\label{app:mi}

\begin{proof}[Proof of \Cref{thm:mi-counterexample}]

\paragraph{Part (a).}
Write inputs as $x=(v_1,\ldots,v_N,y_1,y_2)\in\{0,1\}^{N+2}$.
Under $D$, choose $J$ uniform in $[N]$, set $v_J=1$ and $v_i=0$ for $i\neq J$, and draw $y_1,y_2$ independent and uniform.
For an arbitrary input, let $J(x)=\min\{i:v_i=1\}$, with $J(x)=0$ if $v=0$.
Let $\theta_a(x)=\1\{J(x)\ge a\}$ for $a=2,\ldots,N$, let $\pi(x)=y_1\oplus y_2$, and let $\F=\{\theta_2,\ldots,\theta_N,\pi\}$.

The thresholds are nested, so a dual-shattered set contains at most one of them, and $d^{*}\le2$.
The pair $\{\theta_2,\pi\}$ is dual-shattered, since $v$ can be chosen with $J(x)\in\{1,2\}$ and $(y_1,y_2)$ can make $\pi$ either value; so $d^{*}=2$.

Suppose the selected set consists of thresholds only, and let $Y$ be their labels, a function of $J$.
The bit $\pi(X)$ is independent of $(J,y_1)$ and of $(J,y_2)$, so $I(X_i;Y,\pi(X))=I(X_i;Y)$ for every coordinate $X_i$, and adding $\pi$ does not increase $\Phi$.
If $\theta_a$ is not selected, the values $J=a-1$ and $J=a$ give the same $Y$, both with positive probability.
Conditioned on that value of $Y$, the coordinate $v_{a-1}$ and the label $\theta_a(X)$ both take the value one with positive probability but never simultaneously, so $I(v_{a-1};\theta_a(X)\mid Y)>0$ and adding $\theta_a$ strictly increases $\Phi$.
So the $\Phi$-greedy procedure selects all $N-1$ thresholds before $\pi$.
It cannot stop earlier, because while $\pi$ is unselected, $\pi(X)$ is a uniform bit independent of the representation, and every decoder errs on it with probability $1/2>\eps$.
Hence at least $N$ functions are selected.

For the one-shot statement, let $\mathcal G$ be a $t$-set with $t\le N-1$.
If $\pi\in\mathcal G$, some threshold $\theta_a$ is missing, and replacing $\pi$ by $\theta_a$ strictly increases $\Phi$ by the two computations above.
So no maximizer contains $\pi$, and every maximizer errs on $\pi$ with probability $1/2$.

\paragraph{Part (b).}
Choose $p\in(0,0.01)$ with $\lambda h_2(p)\le1$ and an integer $R$ with $Rh_2(p)>1$.
Let $Y_0$ be a uniform bit and $B_1,\ldots,B_\lambda$ independent Bernoulli$(p)$ bits, independent of $Y_0$.
The input consists of $Y_0$ and $R$ identical copies of each $B_i$, so $n=1+\lambda R$, and $\F$ is the set of coordinate functions.
Then $H(X)=1+\lambda h_2(p)\le2$.

Suppose the selected coordinates reveal some of the $B_i$ but not $Y_0$, and let $Y$ be their values.
Adding $Y_0$ increases $\Phi$ by exactly $I(Y_0;Y_0\mid Y)=1$, since $Y_0$ is independent of everything else.
Adding a coordinate of an unrevealed block $i$ increases $\Phi$ by at least $\sum_{\text{copies}}I(B_i;B_i\mid Y)=Rh_2(p)>1$.
Coordinates of revealed blocks add nothing.
So the $\Phi$-greedy procedure reveals all $\lambda$ rare bits before selecting $Y_0$, and it cannot stop before selecting $Y_0$, because $Y_0$ is a uniform bit independent of the other coordinates.
Hence at least $\lambda+1$ functions are selected.

For the one-shot statement, a set $\mathcal G$ that reveals $\lambda'$ distinct rare bits has $\Phi(\mathcal G)=\1\{Y_0\in\mathcal G\}+\lambda'Rh_2(p)$.
Since $Rh_2(p)>1$, every maximizer with $t\le\lambda$ elements reveals $t$ distinct rare bits and excludes $Y_0$, and therefore errs on $Y_0$ with probability $1/2$.
\end{proof}

\section{Experimental configurations}
\label{app:experiments}
\paragraph{Finite-family experiment.} The Gaussian observation noise has standard deviation \(0.1\). Each instance uses \(500n_{\mathrm a}\) training observations, \(250n_{\mathrm a}\) validation observations, and \(500n_{\mathrm a}\) test observations. The configuration grid is
\[
W\in\{4,8,16,32\},\qquad
k\in\{1,2,4\},\qquad
q\in\{1,2,4\}.
\]
Each of the 36 configurations is trained separately with five random seeds per instance, providing 3,420 training runs. Before quantization, retained activations are clipped at the $99.9$th percentile. We evaluate the error targets
\[
\varepsilon\in\{0.005,0.01,0.025,0.05,0.1,0.2\}.
\]
\paragraph{Threshold-product experiment.} For each dimension considered in \Cref{sec:exp-pilot}, we use 65,536 training observations and 65,536 test observations, shared across three random seeds. We train one autoencoder per seed with \(W=2r\) and \(k=r\), then quantize each trained model at \(q\in\{1,\ldots,5\}\). This provides $\StudyPilotModels$ trained models and $\StudyPilotRepresentations$ quantized representations.

\end{document}